\documentclass[twoside]{article}

\usepackage[accepted]{aistats2026}
\usepackage[round]{natbib}

\def\hpp{\widehat{\pi}'}

\usepackage{graphicx} 
\usepackage{amsmath}
\usepackage{amsthm}
\usepackage{algorithmic}
\usepackage[ruled,vlined,linesnumbered]{algorithm2e}
\usepackage{xcolor}
\usepackage{multicol} 
\usepackage{enumitem}

\SetCommentSty{mycommfont}

\usepackage{color}
\usepackage{subcaption}
\usepackage{booktabs}
\usepackage{pifont}
\usepackage[table]{xcolor}
\usepackage{bbm}
\usepackage{multirow}
\usepackage{url}

\newtheorem{theorem}{Theorem}
\newtheorem{lemma}{Lemma}
\newtheorem{remark}{Remark}

\begin{document}

%

%

\twocolumn[

\aistatstitle{Prior Shift Estimation for Positive Unlabeled Data Through the Lens of Kernel Embedding}

\aistatsauthor{Jan Mielniczuk \textsuperscript{1,3} \And Wojciech Rejchel \textsuperscript{2} \And   Paweł~Teisseyre\textsuperscript{1,3} }



\aistatsaddress{
\textsuperscript{1}Polish Academy of Sciences, \\Warsaw, Poland
\And 
\textsuperscript{2}Nicolaus Copernicus University,\\ Toruń, Poland
\And 
\textsuperscript{3}Warsaw University of Technology,\\ Warsaw, Poland} 
]

\begin{abstract}
We study estimation  of a class prior  for unlabeled target samples which  possibly differs from that of source population. Moreover, it is assumed that the source data is partially observable: 
only samples from the positive class and from the whole population are available (PU learning scenario). We introduce a novel direct estimator of the class prior which avoids estimation of posterior probabilities in both populations and has a simple geometric interpretation. It is based on a distribution matching technique together with kernel embedding in Reproducing Kernel Hilbert  Space and is obtained  as an explicit solution to an optimisation task. We establish its asymptotic consistency as well as an explicit non-asymptotic    bound  on its deviation from the unknown prior, which is calculable in practice. We study  finite sample behaviour for synthetic and real data and show that the proposal works consistently on par or better than its competitors.
\end{abstract}

\section{INTRODUCTION}

Positive Unlabeled (PU) learning \citep{ElkanNoto2008,BekkerDavis2020} is an active research topic that has attracted a lot of interest in the machine learning community in recent years due to a common occurrence of data which is only partially observable. The goal is to build a binary classification model based on training data that contains solely positive cases and unlabeled ones, which can be either positive  or negative. Typical  example is a scenario, when patients with a confirmed diagnosis of a disease are  treated as positive cases, while patients with no diagnosis are considered as unlabeled observations, since this group may include both sick and healthy individuals. PU data occurs frequently in
many fields, such as bioinformatics \citep{Li2021}, image and text classification \citep{LiLiu2003,Fung2006}, and survey research \citep{Sechidis2017}.

Most state-of-the-art learning algorithms for PU data, such as nnPU \citep{Kiryo2017}, VAE-PU \citep{Na2020} and others \citep{SelfPU,Zhao2022,Luo2021_PULNS} require knowledge of the class prior, i.e. the probability of the positive class. Knowledge of the class prior can be used to either obtain calculable representation of  a risk function, or to modify  a threshold  of  a classification rule learned on source data. Since the class prior is usually unknown, there is an important line of research aimed at developing methods for estimating it from PU data, see \cite{Jainetal2016,Ramaswamy2016,BekkerAAAI18} for representative examples. The task is non-trivial, because in the case of PU data we do not have direct access to negative observations, but only to an unlabeled sample which is a mixture of positive and negative observations. This implies that the prior is not identifiable in general and one needs to impose assumptions to ensure its uniqueness. 
Moreover, most of the existing methods assume that the class prior is the same for  the source (training) data and the target (test) data.
This assumption is not fulfilled in many situations. For example, imagine that the source data is collected in the period before the outbreak of an epidemic, where the percentage of people with the considered disease is small, while the target data is gathered during an epidemic, where the prevalence of the disease may be much higher \citep{Rolandaetal2022}. The source and target data may also be collected in different climatic zones, which naturally differ in the prevalence of diseases.
In such situations, it is necessary to estimate the class prior probability not only for the source PU data but also for new unlabeled target data, for which we only observe features whereas the labels remain unknown.
Figure \ref{fig_motivation} illustrates the discussed situation. In the example, the source class prior $\pi=0.2$ whereas the target class prior is $\pi'= 0.8$.

\begin{figure}[ht!]
\centering
    \begin{tabular}{c}
      \includegraphics[width=0.4\textwidth]{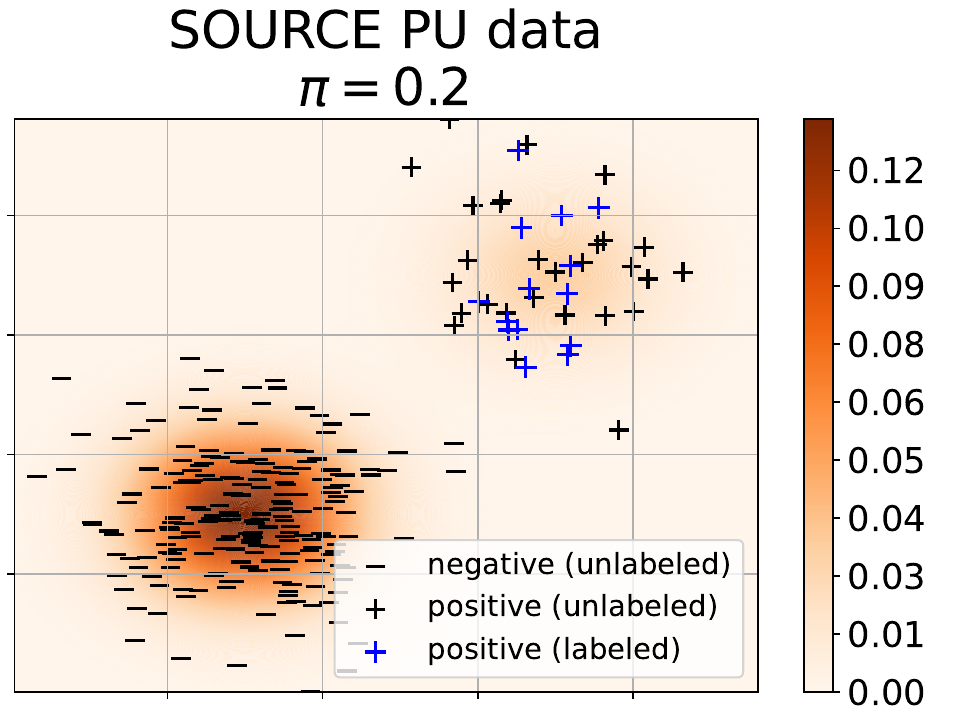} \\
      \includegraphics[width=0.4\textwidth]{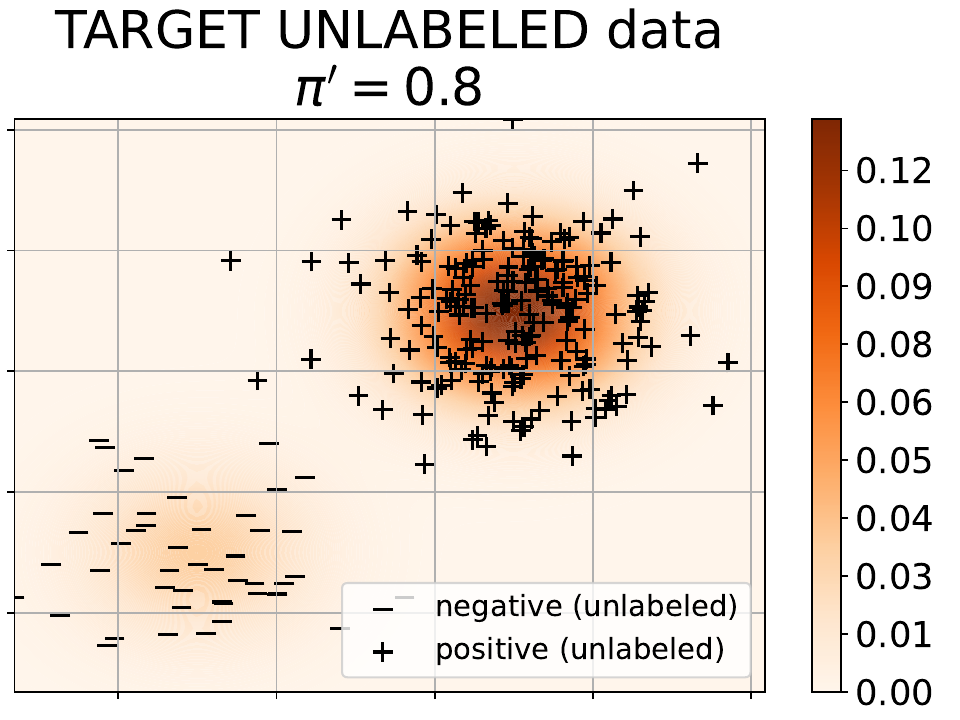} 
      \end{tabular}
    \caption{Label shift visualization for PU data. Source (training) data contains positive (blue) and unlabeled (grey) observations. Target (test) data contains only unlabeled observations. The class priors differ between source data ($\pi=0.2$) and target data ($\pi'=0.8$). The goal is to estimate the target class prior $\pi'$ using source PU data and target unlabeled data.}
    \label{fig_motivation}
\end{figure}

The problem of inference under class prior shift, known in the literature as label shift, has been extensively studied and several methods have been developed to estimate the probability of the class prior for the target data as well as to modify the classifier to take the shift into account \citep{SaerensLatinneDecaestecker2002,LiptonWangSmola2018,Gargetal2020,Iyer2014,Vaz2019}. We note that in business applications evaluation of proportion of each  label on unlabeled data set (known as quantification task) is frequently even more needed than classification itself. This is particularly important for applications tracking trends (see \cite{Forman2008} and \cite{Gonzales2017} for the review). 
However,  methods developed for label shift problem require fully labeled source data and thus cannot be directly applied to the problem considered here.  Whence the important endeavour is to
estimate the  target class prior directly, possibly avoiding label prediction for  source and target samples.
Surprisingly,  despite its obvious importance, this problem has been rarely discussed. The only example up to now is \cite{Nakajima23} where an approach based on estimation of ratio of densities for positive and mixture distribution is proposed. Its  authors, however, focus mainly on  properties of the ensuing classifier, not on the properties of label shift estimator.

The main contribution of this work is the proposal of a new estimator of the target class prior $\pi'$  called \textbf{TCPU} for which label prediction for target population is not necessary. Our approach is based on employing the distributions matching technique and the kernel method. 
The kernel method \citep{Fukimizu2007,Gretton2012} is a powerful approach successfully applied for many problems of machine learning due to  related universality property. For  shifted prior estimation when the data is fully observable it has been used in \cite{Iyer2014}.
Here, we  provide a rigorous treatment of proposed estimator, including its  consistency, i.e. convergence in probability to the true target class prior. Even more importantly, we provide a {\it non-asymptotic} bound on the approximation error.
Additionally, in the paper we show how to adapt the popular KM estimator \citep{Ramaswamy2016}, being a state-of-the-art class prior estimator for PU data, to the  case of  class prior shift.
Our experiments, conducted on artificial and real data for different class prior shift schemes, confirm the effectiveness of  proposed method TCPU.

\section{LABEL SHIFT FOR PU LEARNING}
We first introduce  relevant notation. Let $X\in {\cal X}$ be a random variable corresponding to a feature vector, $Y\in\{-1,1\}$ be a true class indicator and $P_{XY}$  their joint distribution (called a source distribution). We consider a problem of modeling Positive Unlabeled data in case-control setting,  which means that only samples coming from the positive class and samples coming from the overall population  $X$ are available. More formally, let $P_X$ be the distribution of $X$ and $P_+=P_{X|Y=1}$, $P_-=P_{X|Y=-1}$ the distributions of samples from the positive class and the negative class, respectively. We denote by $X_1,\ldots,X_n$ independent samples generated according to $P_X$ and  $X_1^+,\ldots,X_m^+$ generated according to $P_+$.

Moreover, we consider  the second vector $( X',Y')$ such that its distribution $P'_{X'Y'}$ (called a target distribution) is a label shifted distribution of $(X,Y)$, which means that the marginal distribution of $Y'$ is different from that of $Y$, i.e.
\[  \pi'=P'(Y'=1)\neq P(Y=1)=\pi,\]
however, the covariate distributions in both the  positive and the negative class remain the same:
\begin{equation}
\label{LSassumption}
P'_{ X'| Y'=i}= P_{ X|Y=i}\quad {\rm for} \quad i=\pm 1.  
\end{equation}
Thus, we have that a distribution of $X'$ satisfies  $  P'_{ X'}= \pi' P_+ + (1- \pi') P_{-}$ and  is different from $P_X=
\pi P_+ + (1- \pi) P_{-}$.

Denote by 
$ X_1',\ldots,  X_{ n'}' $ independent samples generated from $ P'_{X'}$. In the paper, we consider the problem of estimation of label shifted probability $\pi'$. Note that   the problem is nontrivial as the class indicators corresponding to shifted samples are not available. However, we have at our disposal data  from the positive class and unlabeled $X$ observations corresponding to a different prior $\pi$.

We note that if $\pi$ is known the distribution $P_{-}$ is  uniquely determined and the problem is well
defined in general.  When  $\pi$  is unknown,  the specific assumptions are needed under which it is identifiable, e.g. an assumption  that $P_-$ is {\it not} a convex combination of $P_+$ and other probability measure  (see e.g. \cite{Ramaswamy2016}).

Finally, it is worth mentioning that estimation of $\pi'$ is crucial to define a classification rule for the target set which involves a modified threshold based on $\pi'$ \citep{LiptonWangSmola2018}:
\begin{align*}
&P'(Y'=1|X')>0.5 \iff \\
&P(Y=1|X)>\frac{(1-\pi)\pi'}{(1-\pi')\pi+(1-\pi)\pi'}.
\end{align*}
A natural approach is to train PU classifier, such as nnPU \citep{Kiryo2017},  on the source data, to estimate $P(Y=1|X)$, and then apply the above decision rule.

\section{CLASS PRIOR ESTIMATION FOR TARGET DATA}
\subsection{TCPU: a novel kernel-based estimator}
\label{Sec:TCPU}
In this section we introduce a novel method of estimating $\pi'$,  which will be called {\bf TCPU} (a {\bf T}arget {\bf C}lass prior estimator for {\bf P}ositive-{\bf U}nlabled data under a label shift).   

Let $K(\cdot,\cdot)$  be a kernel function (i.e. a symmetric, continuous and semi-positive function defined on $\cal X\times \cal X$) and let $\cal H$ be a Reproducing Kernel Hilbert Space (RHKS)  induced by $K(\cdot,\cdot)$ (see e.g. \cite{Fukimizu2007}).
We denote an associated  kernel transform by  $\phi(x)=K(x,\cdot)\in {\cal H}.$  A scalar product in $\cal H$  of $\phi(x_1)$ and $\phi(x_2)$ is defined as  $<\phi(x_1), \phi(x_2)>_{\cal H} = K(x_1,x_2)$ for  $x_1,x_2 \in \cal X$ and  is naturally extended for general elements of ${\cal H}$.
Next, let  
$ \Phi(P_X)=\mathbbm{E}\phi(X)=\int_{\cal X} \phi(s)\,P_X(ds)\in {\cal H}$ 
be a  mean functional of $P_X$ with a norm $||\Phi(P_X)||_{\cal H}^2= \mathbbm{E} K(X_1,X_2),$ where $X_1$ and $X_2$ are two independent random vectors following a distribution $P_X.$  The mean functionals $\Phi(P_+)$ and $\Phi(P'_{X'})$ are defined analogously. Recall that when kernel is universal we have that $\Phi(P)=\Phi(Q)$ is equivalent to the fact  that distributions  $P$ and $Q$ coincide \citep{Fukimizu2007}.

Let $\hat{P}_X$,  $\hat P_+$ and 
$\hat{ P'}_{X'}$  be
empirical  distributions corresponding to observable samples. In  the following we will omit indices $X$ and $X'$ in $P_X$ and $P'_{X'}$, respectively, and the same convention is applied to their  empirical counterparts. We note that due to the fact that $\hat P$ is a discrete distribution with mass $n^{-1}$ at each observation $X_i$ we have that $\Phi(\hat P)=n^{-1}\sum_{i=1}^n\phi(X_i)$ and analogous representations hold for $\Phi(\hat P_+)$  and $\Phi(\hat P')$.

In the above setup we have that
\begin{align*}
(1-\pi')(P-\pi P_+) &= (1-\pi')(1-\pi)P_-\\
&=(1-\pi)( P' -\pi' P_+).
\end{align*}
Therefore, in order to determine $\pi'$, it is natural to substitute $\gamma$ for $\pi'$ and minimize the following objective function
\begin{equation}
\label{Lfunc}
\mathcal{L}(\gamma) =
||(1-\gamma)A - (1-\pi)B(\gamma)||^2_{\cal H},  
\end{equation}
$A=[\Phi(P)-\pi \Phi(P_+)]$ and $B(\gamma)=[\Phi( P') -\gamma \Phi(P_+)]$.
\begin{lemma}
Suppose that a kernel $K$ is universal, 
$P_+\neq P_-$ and $\pi<1$. Then $\pi'$ is unique minimizer of $\mathcal{L}(\gamma)$.
\end{lemma}
\begin{proof}
Denote by $Crit(\gamma)$ the function given by
\begin{equation}
    \label{crit2}
     Crit(\gamma) =|\pi'-\gamma|\times(1-\pi)||\Phi(P_-)- \Phi(P_+)||_{\cal H}.
\end{equation}
Then by noting that
\begin{align*}
&\Phi(P) -\pi\Phi(P_+)=(1-\pi)\Phi(P_-),\\
&\Phi(P') -\gamma\Phi(P_+)=(\pi'-\gamma)\Phi(P_+) +(1-\pi')\Phi(P_-)    
\end{align*}
we have that $\mathcal{L}(\gamma)= [Crit(\gamma)]^2$  and the minimiser of $Crit(\gamma)$ is obviously $\pi'$.    
\end{proof}

In the proposed method, we consider the empirical version of (\ref{Lfunc}) given as
\begin{equation}
\label{Lfunc_emp}
\widehat{\mathcal{L}}(\gamma) =
||(1-\gamma)\widehat{A}- (1-\pi)\widehat{B}(\gamma)||^2_{\cal H},
\end{equation}
$\widehat{A}=[\Phi(\hat P)-\pi \Phi(\hat P_+)]$ and $\widehat{B}(\gamma)=[\Phi(\hat{ P'}) -\gamma \Phi(\hat P_+)]$

and the estimator of $\pi'$ is defined as its minimizer 
\begin{equation}
\label{TCPU}
\hpp ={\rm argmin}_\gamma \widehat{\mathcal{L}}(\gamma). 
\end{equation}
The estimator defined above will be  called  {\bf TCPU}  further on. Figure \ref{fig_objective}  illustrates the behavior of the objective function and TCPU estimator.
\begin{figure*}[ht!]
\centering
    \begin{tabular}{c c c}
    \includegraphics[width=0.3\textwidth]{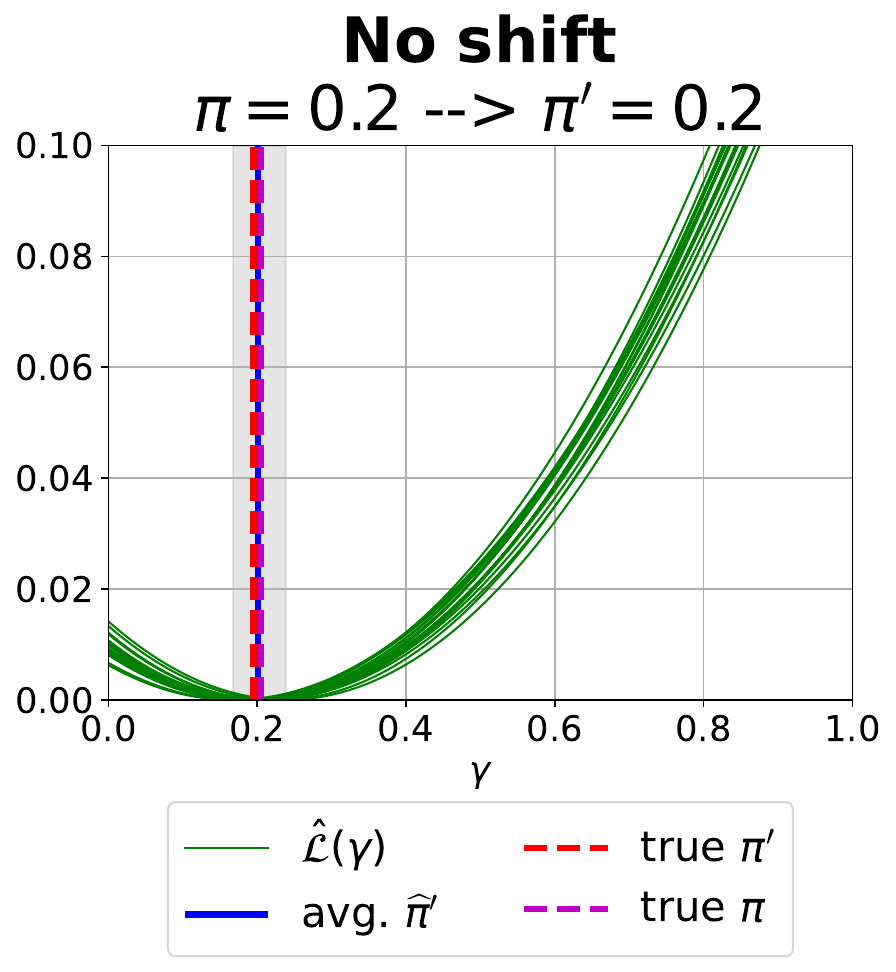} &
      \includegraphics[width=0.3\textwidth]{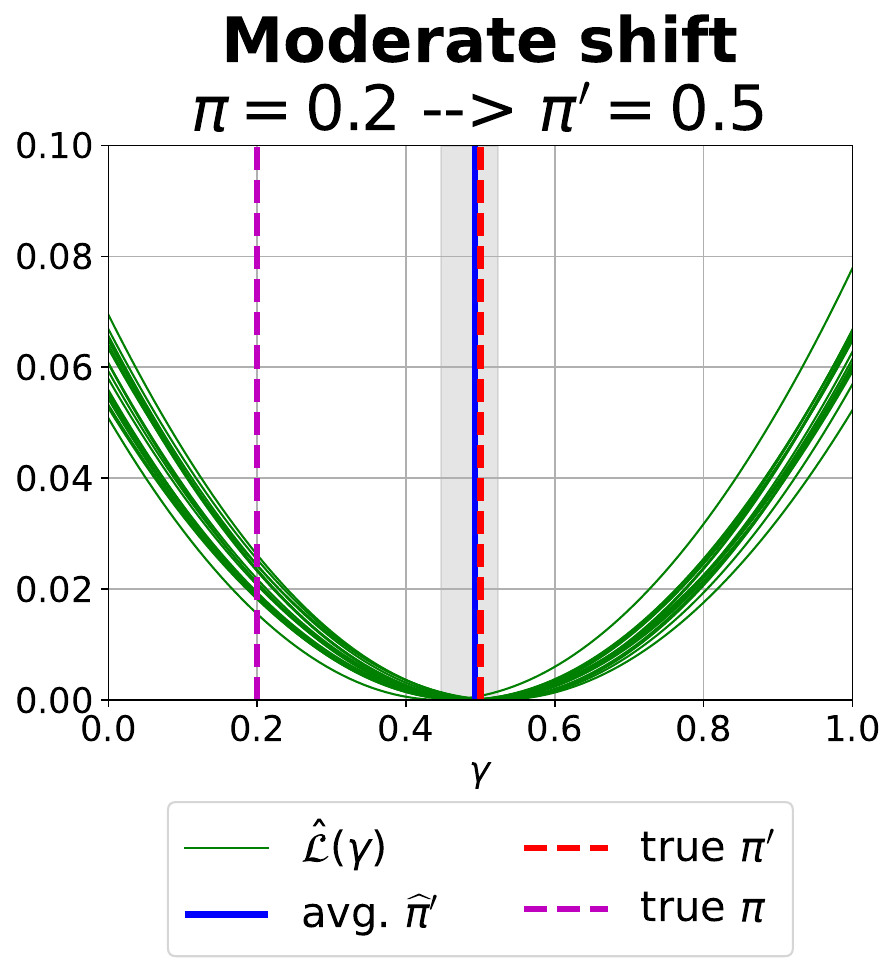} &
      \includegraphics[width=0.3\textwidth]{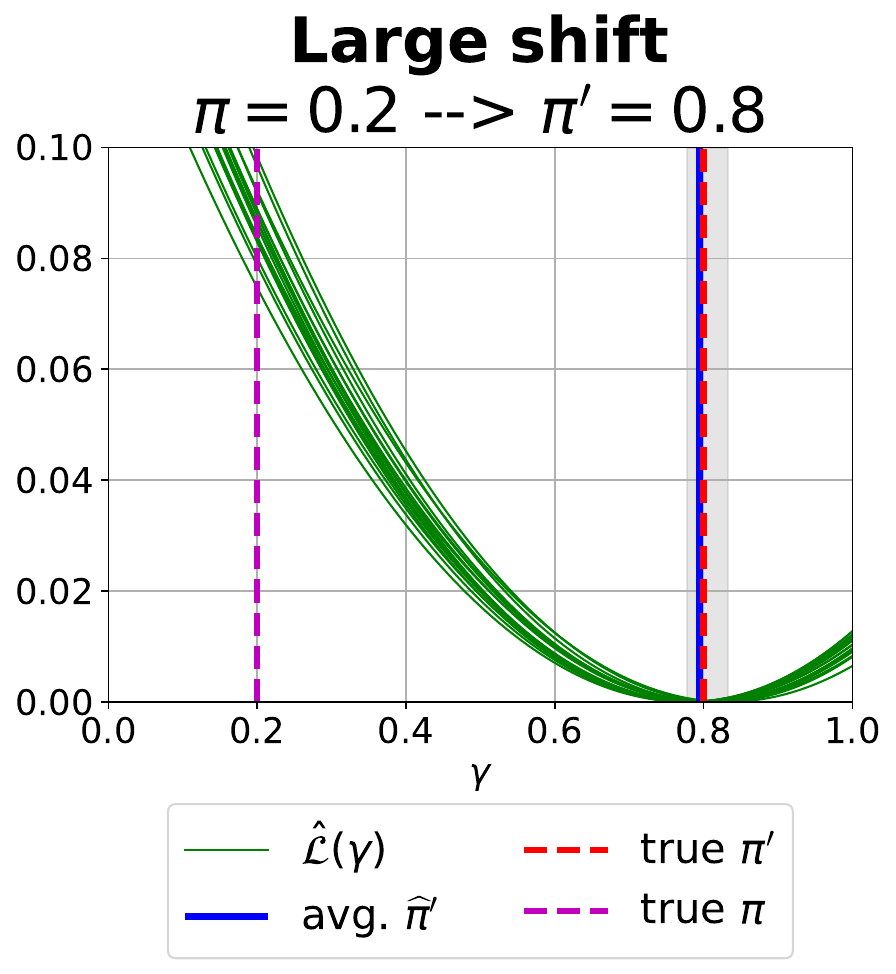}  \\
      \end{tabular}
       \caption{Visualization of the  objective function behavior for $\pi=0.2$. Three cases are considered: $\pi'=0.2$ (no shift), $\pi'=0.5$ (moderate shift) and $\pi'=0.8$ (large shift). The TCPU estimator is defined as $\hpp=\arg\min_{\gamma}\hat{\mathcal{L}}(\gamma)$. The grey area indicates the range of estimator's values  for $20$ runs.}
    \label{fig_objective}
\end{figure*}

For theoretical results  it will be assumed that $\pi$ is known. This assumption is plausible when the large data base corresponding to source distribution is available.

In the following lemma we show that the proposed estimator can be  explicitly calculated using a simple algebraic formula.
\begin{lemma}
 Let $\hpp$  be defined by (\ref{TCPU}). Then we have
 \begin{align}
    \label{explicit}
   &\hpp=\frac{<\Phi(\hat P) -\Phi(\hat P_+), \Delta>_{\cal H}}{||\Phi(\hat P) -\Phi(\hat P_+)||_{\cal H}^2}=\\
   &1-\frac{(1-\pi)<\Phi(\hat P) -\Phi(\hat P_+),\Phi(\hat{P'})- \Phi(\hat P_+)>_{\cal H}}{||\Phi(\hat P) -\Phi(\hat P_+)||_{\cal H}^2},\nonumber 
 \end{align}
 where $\Delta=\Phi(\hat P) -\pi\Phi(\hat P_+) -(1-\pi)\Phi(\hat{P'}).$
\end{lemma}
\begin{remark}
In view of the first equality in (\ref{explicit}),   $ \hpp$  has  a simple geometric interpretation: it is a coefficient of projection of a function $\Delta$ on  $\Phi(\hat P) -\Phi(\hat P_+)$ in $ {\cal H}$.
\end{remark}
\begin{proof} Simple calculations show that the expression under the norm in (\ref{Lfunc_emp}) equals
\begin{align*}
&-\gamma[\Phi(\hat P) -\Phi(\hat P_+)] + \Phi(\hat P) -\pi\Phi(\hat P_+) -(1-\pi)\Phi(\hat{P'})\\
&=-\gamma[\Phi(\hat P) -\Phi(\hat P_+)] + \Delta.    
\end{align*}

Thus,  the squared norm equals
\begin{align*}
&\gamma^2||\Phi(\hat P) -\Phi(\hat P_+))||_{\cal H}^2  +\\
&2\gamma<\Phi(\hat P_+)-\Phi(\hat P),\Delta >_{\cal H} + ||\Delta||_{\cal H}^2    
\end{align*}
and the first form of $\hpp$ follows  by calculating a minimizer of the above function. The second formula follows by simple algebraic manipulations.
\end{proof}
The kernel approach is a core of the Maximum Mean Discrepancy (MMD) method which matches the distributions based on the features in RHKS induced by a  kernel $K$. It is
frequently used in machine learning and statistics to compare  a data  distribution with a specific distribution, in  a two-sample problem and covariate shift detection (see \cite{Gretton2012}). 
It has been also applied for the  label shift problem in the classical framework.
 Namely, for the case  when  distribution $P_{XY}$ is observable  the approach  relies on the equality \citep{Zhang2013}
\begin{align}
    \label{kernel_LS}
     &\pi'= \\
     &{\rm argmin}_{\lambda\in[0,1]}|| \Phi(P_{X'}) -[\lambda \Phi(P_+) + (1-\lambda)\Phi(P_-)]||^2_{\cal H}\nonumber
\end{align}
(see also \cite{Iyer2014} and \cite{Dussap2023}). 
In the considered scenario, a direct application of (\ref{kernel_LS}) to estimate $\pi '$ is clearly  infeasible, as observations from the negative class are not available.
The estimator (\ref{TCPU}) can be considered as a modification  of  the MMD approach applied  for a label shift in  PU setting.

\subsection{Asymptotic consistency  and non-asymptotic error bounds for the proposed estimator}
We let $N=\min(n,m, n')$. In the next result we establish asymptotic and non-asymptotic error bounds for $\hpp.$ In its proof we also state conditions, which guarantee that $||\Phi(\hat P) -\Phi(\hat P_+)||_{\cal H} >0,$ which is implicitly assumed  in~\eqref{explicit}. The proofs of all results below are given in the supplement.

\begin{theorem}
\label{th_cons}
 Suppose that a kernel $K$ is universal, 
$P_+\neq P_-$ and $\pi<1.$\\
(i) Moreover, assume   $\mathbbm{E} K(X,X)<\infty$ for $X\sim P$ and analogous conditions hold  for $X^+\sim P_{X|Y=1}$ and $X'\sim P '$. Then for $N\to\infty$ we have  $\hat{\pi}'\to \pi'$ in probability.\\
(ii) Assume that $M=\sup_x K(x,x)<\infty.$ Moreover, fix $\alpha\in (0,1)$ and $\delta \leq \exp(-(\sqrt{2}+1)^2/2)$ and let
\begin{equation}
    \label{assum}
N\geq \frac{16 M\log(1/\delta)}{(1-\alpha)^2(1-\pi)^2||\Phi(P_-)-\Phi (P_+)||^2_{\cal H}} \;.
\end{equation}
 Then we have
\begin{equation}
    \label{th_claim}
P\left(|\hpp- \pi'|\leq \frac{4 \sqrt{\frac{M}{N}\log(1/\delta)}}{\alpha(1-\pi)||\Phi(P_ -)-\Phi (P_+)||_{\cal H}}\right) \geq 1-3\delta. 
\end{equation}
\end{theorem}

We note that consistency of $\widehat{\pi}'$ is proved in Theorem \ref{th_cons}(i) under weak conditions.
Indeed, dropping the conditions   $P_+ \neq P_ -$ and $\pi\ne 1$  makes the problem ill-posed. Finally, the restrictions imposed on a kernel are satisfied, for instance, for a gaussian kernel.

The claim  of  Theorem \ref{th_cons} (ii) is stronger (it is a nonasymptotic result implying asymptotic consistency), so it needs more restrictive assumptions as well. However, these conditions are reasonable as in (i). For instance,  a gaussian kernel satisfies the assumption in (ii) with $M=1.$
 The dependence of an error bound on $N,\delta, \pi, ||\Phi(P_ -)-\Phi (P_+)||_{\cal H}$ is stated explicitly  in \eqref{th_claim}.

\begin{theorem}
    \label{th_random}
Assume that $M=\sup_x K(x,x)<\infty.$ 
 Then for any $\delta \leq \exp(-(\sqrt{2}+1)^2/2)$ we have that
\begin{equation}
    \label{th_claim_random}
P\left(|\hpp- \pi'|\leq \frac{4 \sqrt{\frac{M}{N}\log(1/\delta)}}{|| \Phi(\hat P)-\Phi ( \hat P_+)||_{\cal H}}\right) \geq 1-3\delta. 
\end{equation}
    
\end{theorem}

We stress the differences between Theorems \ref{th_random} and \ref{th_cons}. Probability inequality (\ref{th_claim_random})  does not require  assumption \eqref{assum} and it yields  a bound  on an estimation error which depends on $|| \Phi(\hat P)-\Phi ( \hat P_+)||_{\cal H}.$ Notice that this bound can be calculated in practice. On the other hand, the error bound in \eqref{th_claim} is nonrandom and explicitly establishes the dependence on $\pi$ and a distance between $P_ -$ and $P_ + .$

Till now we have assumed that source prior $\pi$ is known. In the experiments we estimate $\pi$ using a well-known KM2 method \citep{Ramaswamy2016} described below  and then plug-in it into (\ref{Lfunc_emp}).  We stress that  this requires additional assumptions  implying that $\pi$ is identifiable, which are imposed for all PU-based approaches.  The most common one is that the distribution of negative samples is {\it not} a mixture of distribution of positive samples and some other probability distribution \citep{Ramaswamy2016, Blanchard2010}.

\section{RELATED WORK}
\label{Sec:Related_work}

\subsection{Method DRPU}
The most related method is DRPU \citep{Nakajima23}, which  in the label-shift setting considered here, constructs empirical Bayes classifier for samples from label shifted population. Other methods consider similar but different scenarios such as PU covariate shift \citep{Sakai2019} and positive data shift \citep{Hammoudeh2020}.
DRPU involves estimator of ratio $r$ of densities of  distributions  $P_+$  and  $P$ for training data and of both prior probabilities $\pi$ and $\pi'$. Estimator $\hat r$ of $r$ is a minimiser of expected Bregman divergence functional (cf Section 2.5 in  \cite{Nakajima23}) and both prior estimators are based on an observation (cf \cite{Blanchard2010}) that $\pi$ and $\pi'$ can be recovered in PU setting by minimising $P(A)/P_+(A)$ (respectively $P'(A)/P_+(A)$) over all sets $A$ for which $P_+(A)>0$. In \cite{Nakajima23},  ratio $\hat{P}'(A)/\hat P_+(A)$ is minimised over sets being the level sets of $\hat r$.  The estimator based on the above method will be denoted simply as {\bf DRPU}.\\
\subsection{Adapting the KM method to label shift}
\label{KM-LS}
As a benchmark, we also introduce  a modification of KM2 estimator \citep{Ramaswamy2016}  for label-shift PU setting.
The KM2 estimator is based on the observation that distribution of negative class $P_-$ can be written as $P_-=\lambda P +(1-\lambda) P_+$, where $\lambda=1/(1-\pi)$ and thus, analogously to (\ref{kernel_LS}), estimator of $\lambda$ can be constructed by projecting $\gamma \phi(P) +(1-\gamma) \phi(P_+)$  on a convex $\mathcal{ C}$  hull of  $\phi(X_1),\ldots,\phi(X_n),\phi(X_1^+),\ldots,\phi(X_m^+) $ and  defining $\hat \lambda$ as $\gamma$ yielding the smallest distance. This leads to two estimators of $\pi$, KM1 and KM2 investigated in  \cite{Ramaswamy2016}. As unavailable positive samples for test population have the same distribution as positive observations from training distribution,
we  can apply  approach from \cite{Ramaswamy2016}  to samples $X'_1,\ldots,X'_{n'},X_1^+,\ldots,X_m^+$  and obtain KM2 estimator of $\pi'$ also investigated below. The adaptation will be called {\bf KM2-LS} in the following. Note that as KM2-LS does not use available observations  from $P$ at all, it is not expected to work well; this observation will be confirmed in our experiments.

\section{EXPERIMENTS}

\subsection{Methods}
We empirically evaluate the effectiveness of TCPU \footnote{The source code: \url{https://github.com/teisseyrep/TCPU}} to recover the true target class prior $\pi'$. 
As baselines we used DRPU \citep{Nakajima23} and KM2-LS methods described in Section \ref{Sec:Related_work}. In the case of DRPU, we utilised the implementation made publicly available by the authors. In the case of KM2-LS, we applied the code of the standard KM2 method \citep{Ramaswamy2016} and adopted it to our setting as described in Section \ref{Sec:Related_work}. Note that TCPU   requires knowledge of $\pi$. Since typically $\pi$ remains unknown, we estimate it using KM2 estimator \citep{Ramaswamy2016}. Importantly, the KM2-LS method does not require the $\pi$ estimation.
The DRPU method is the only one that requires learning a parametric model. As in the original work, we considered Multi-Layer Perceptron (MLP) to this end. Technical details about the model used in DRPU and the selection of hyperparameters are discussed in the supplement. For TCPU, we used Gaussian kernel $K(x,y)=\exp(-\tau||x-y||^2)$
with the default value of parameter $\tau=1/p$, where $p$ is the number of features.

\subsection{Datasets}

The experiments were conducted on 11 datasets, including one synthetic dataset, 3 image datasets: CIFAR-10, MNIST, and FASHION \citep{PyTorch19} and 7 tabular datasets from the UCI repository: Diabetes, Spambase, Segment, Waveform, Vehicle, Yeast and Banknote.
For the synthetic dataset, negative observations are generated from a 10-dimensional normal distribution $N(0, I)$, and positive observations are generated from $N(a, I)$, where $a = (1, \ldots, 1)$.
The characteristics of the UCI and image datasets are provided in Table \ref{Tab:datasets}.
Tabular datasets with multiple classes were transformed into binary classification datasets, where the most common class is treated as the positive class, and the remaining classes are combined into the negative class. For image datasets, the binary class variable is defined according to the specific dataset, following methods used in other PU learning papers \citep{Kiryo2017,Gong2021,Nakajima23}.
For MNIST, even digits form the positive class, and odd digits form the negative class. In CIFAR-10, vehicles form the positive class, and animals form the negative class. For FASHION, clothing items worn on the upper body are marked as positive cases, and the remaining items are assigned to the negative class.
For image data, we use a pre-trained deep neural network, ResNet18, to extract the feature vector. For each image, the feature vector, with a dimension of 512, is the output of the average pooling layer. From the extracted 512-dimensional feature vector, we select the 30 most correlated features with the class variable to reduce the dimensionality of the problem.

\begin{table}[ht!]
\begin{center}
\caption{Statistics of the considered data sets.}
\label{Tab:datasets}
\begin{tabular}{l|lllll}
\toprule 
                Dataset    & $n$ &   $p$ & $P(Y=1)$ & type\\
                \midrule
                  Diabetes &   768  &  8  &0.35 &   tab\\
                  Spambase &  4601  & 57  &0.39 &   tab\\               
                  Segment  & 2310   &19  &0.14  &    tab\\
            Waveform  & 5000   &40  &0.34  &   tab\\
                    Yeast  & 1484   & 8  &0.31  &    tab\\
                  Vehicle  &  846  & 18  &0.26 &     tab\\
                Banknote  &  1347  & 4  &0.44 &     tab\\
\midrule
                  CIFAR10  & 50000   &  - & 0.4    & img\\
                  MNIST  &  60000  & -  & 0.49  & img\\
                  Fashion  & 60000   & -  & 0.5     & img\\                  
\bottomrule
\end{tabular}
\end{center}
\end{table}

\subsection{Experimental settings}

First, each dataset is split into a source and a target dataset in equal proportions. For image data, we use the splits defined in the PyTorch library \citep{PyTorch19}.
For synthetic datasets, we generate observations for fixed values of $\pi$ and $\pi'$.
For real datasets, we simulate a label shift scenario using the downsampling technique. Specifically, we randomly remove observations from one of the classes in both the source and target datasets to control the class priors $\pi$ and $\pi'$, respectively.
Finally, based on the source data, we artificially create a PU dataset by selecting some positive observations for the labeled subset, while the unlabeled subset consists of a mixture of positive and negative observations. 
We follow a known  procedure 
to control the size of the source data and the labeling frequency $c$, which represents the percentage of labeled observations among all positive observations. 
The details of the procedure are provided in the supplement. 

In the experiments, we consider $c=0.25,0.5$ for sythetic and image datasets and $c=0.5$ for UCI datasets. The entire target dataset is treated as unlabeled.
The considered methods take as input the source PU data and unlabeled target data. For each method, we calculate the absolute estimation error $|\pi' - \hpp|$, where $\hpp$ is the estimator returned by the method.
We perform 20 repetitions of the above procedure and analyze the distributions of the errors as well as the distributions of the estimators themselves.

\begin{figure*}[ht!]
\centering
    \begin{tabular}{c c c c}
    \includegraphics[width=0.3\textwidth]{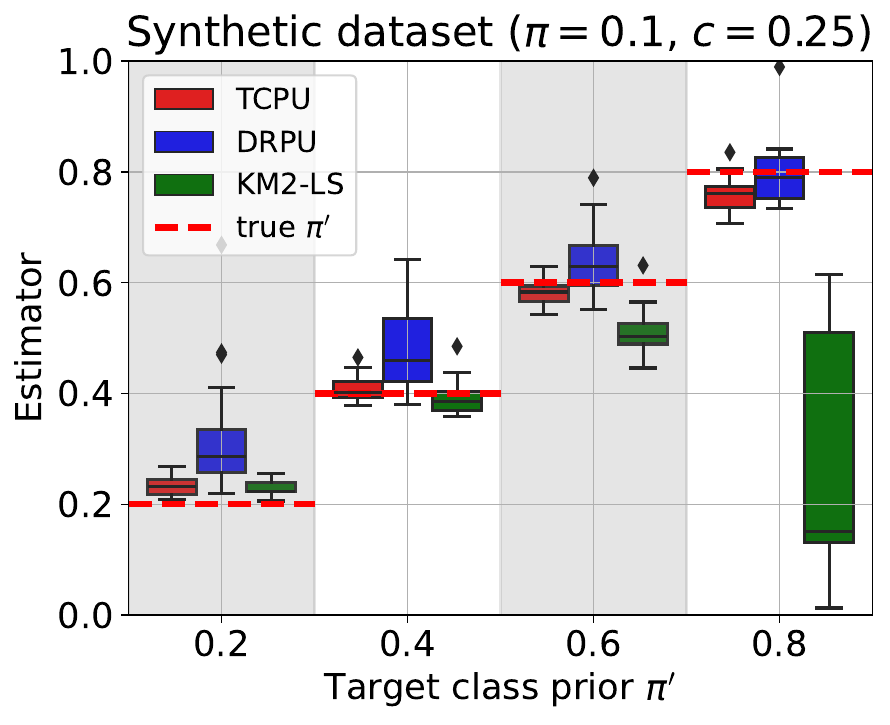}  &
  \includegraphics[width=0.3\textwidth]{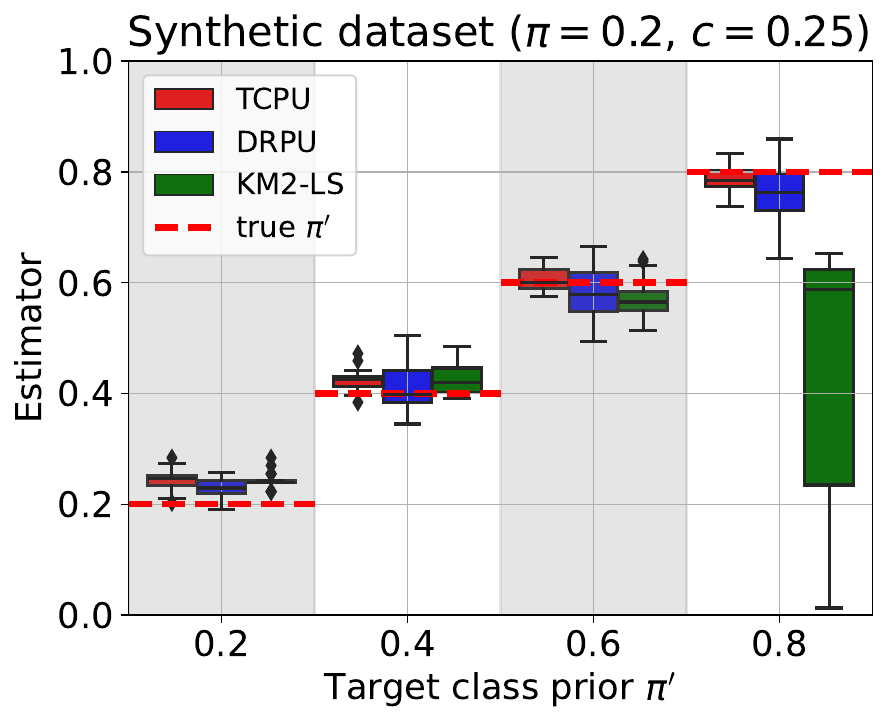} &
          \includegraphics[width=0.3\textwidth]{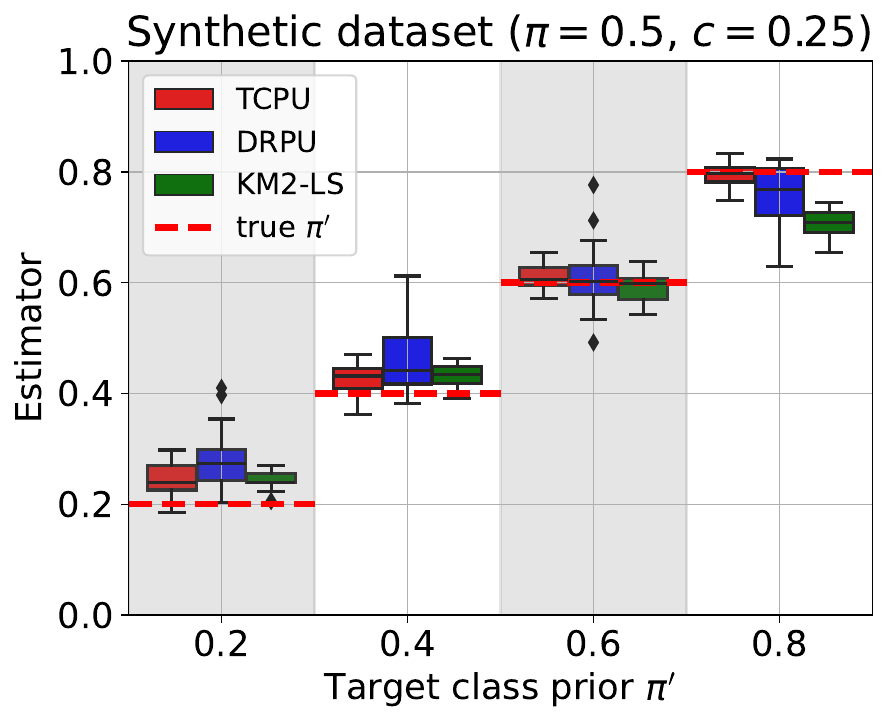}
      \end{tabular}
    \caption{Distribution of estimators (red line indicates  the true $\pi'$). Size of the source data and the target data is $2000$.}
    \label{boxplot_artificial_1_main}
\end{figure*}

\begin{figure*}[ht!]
\centering
    \begin{tabular}{c c c}
    \includegraphics[width=0.3\textwidth]{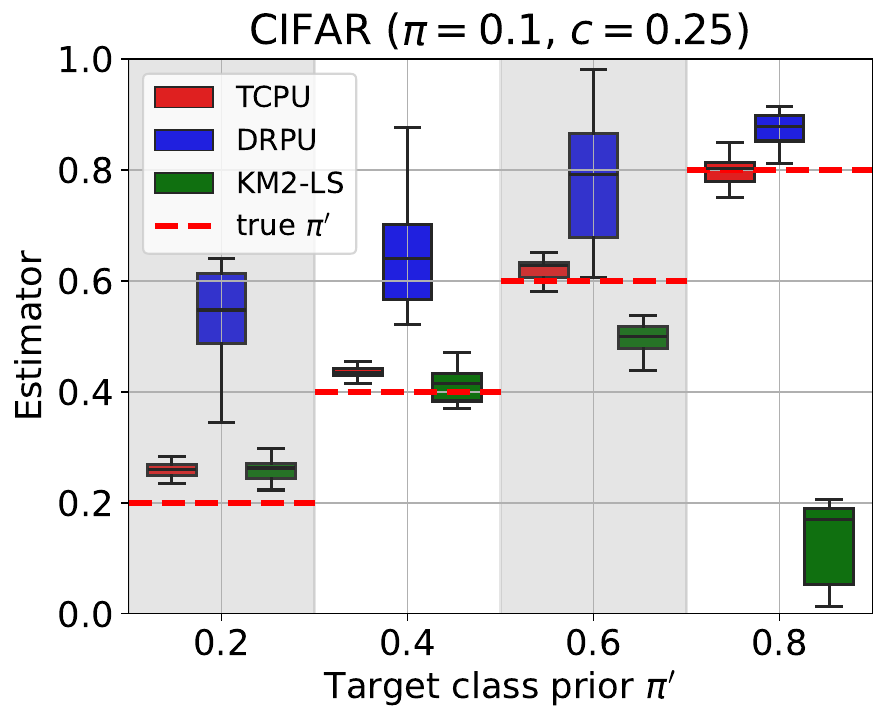}  &
    \includegraphics[width=0.3\textwidth]{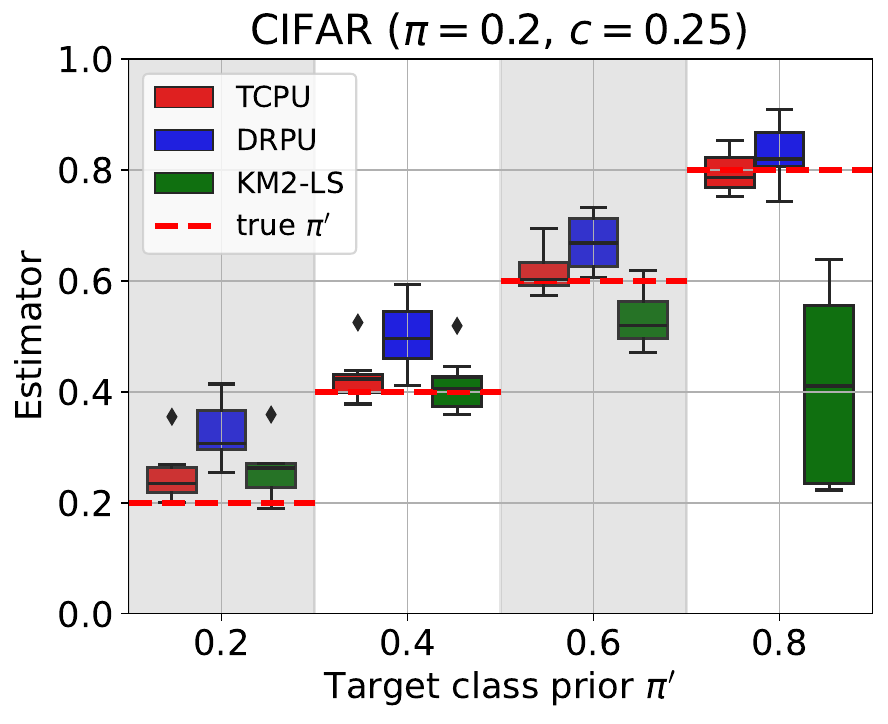}  &
\includegraphics[width=0.3\textwidth]{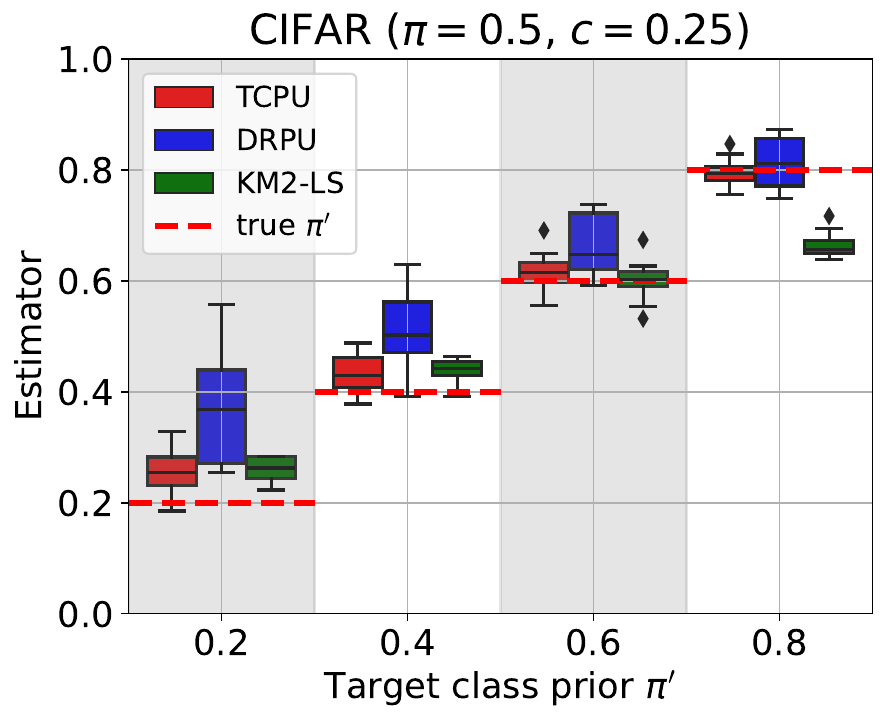}\\
    \includegraphics[width=0.3\textwidth]{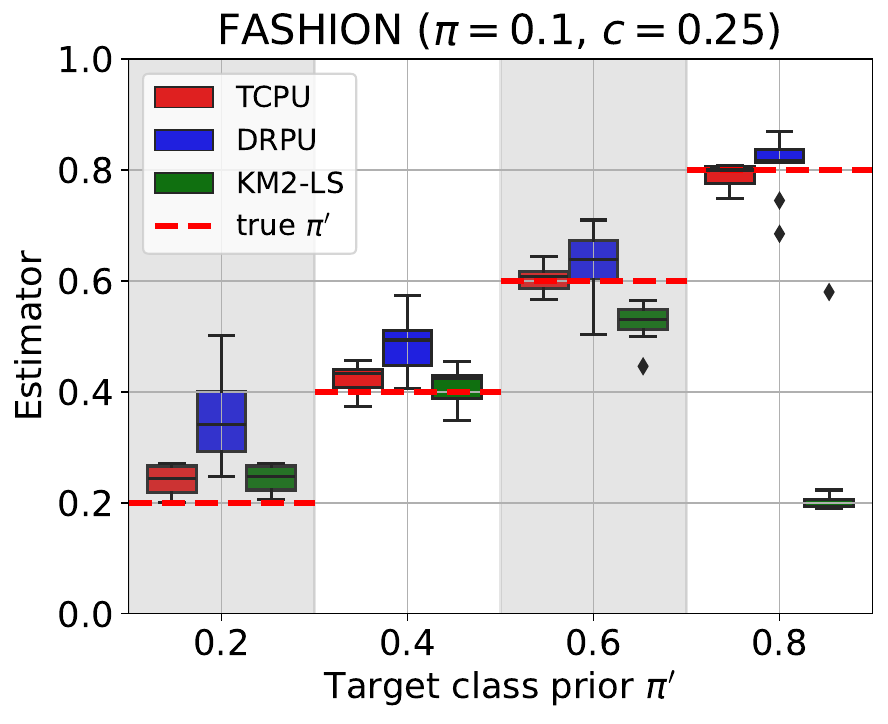}  &
    \includegraphics[width=0.3\textwidth]{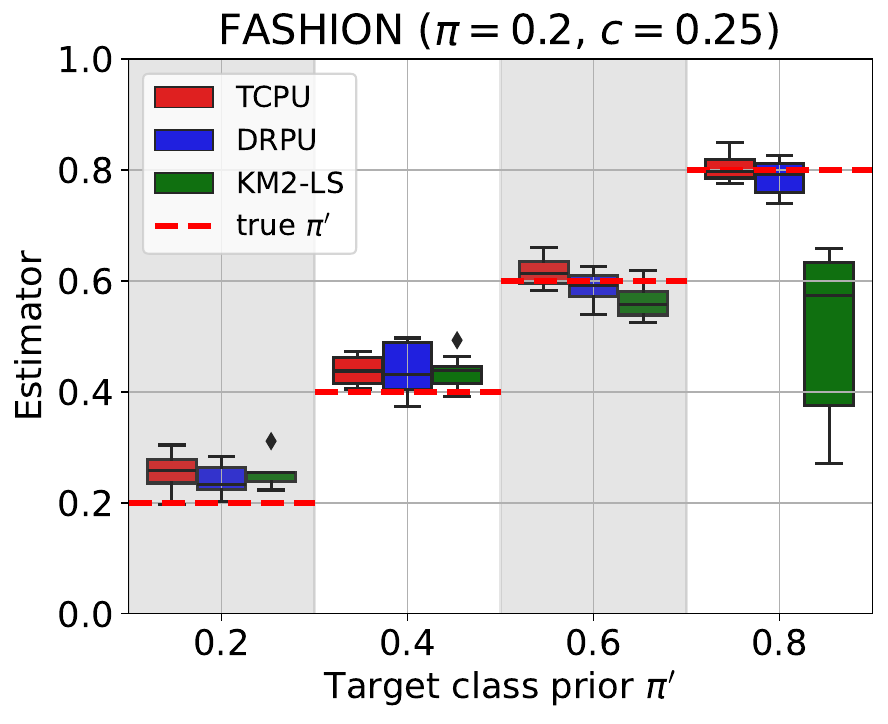}  &
\includegraphics[width=0.3\textwidth]{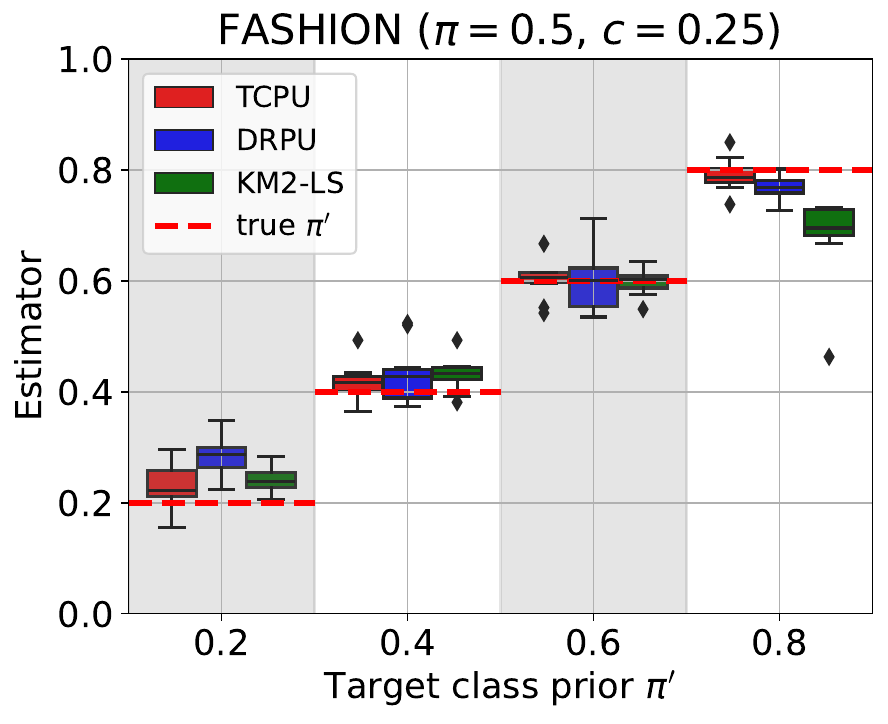}\\
    \includegraphics[width=0.3\textwidth]{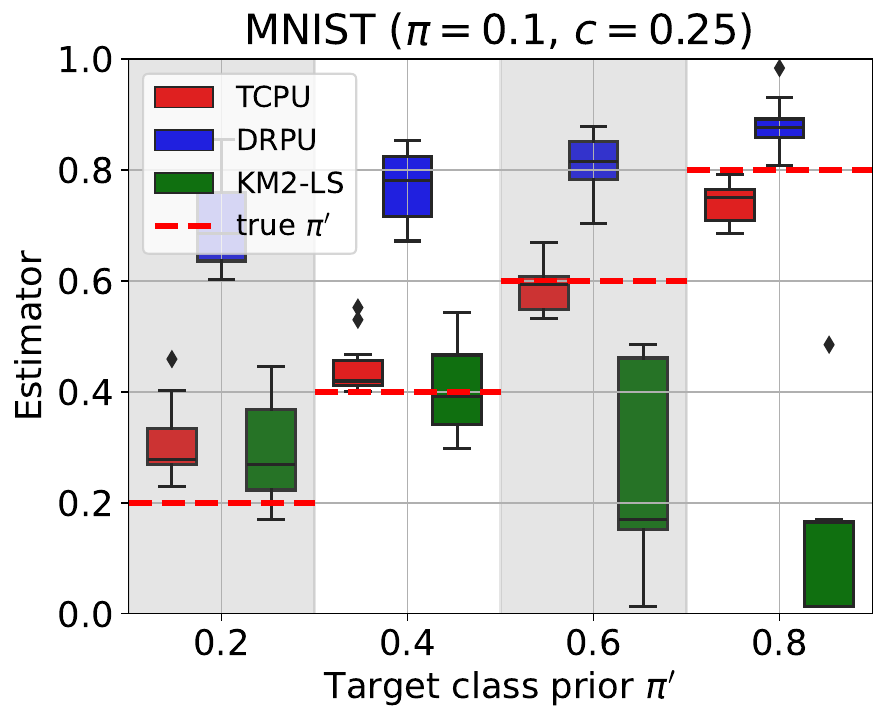}  &
    \includegraphics[width=0.3\textwidth]{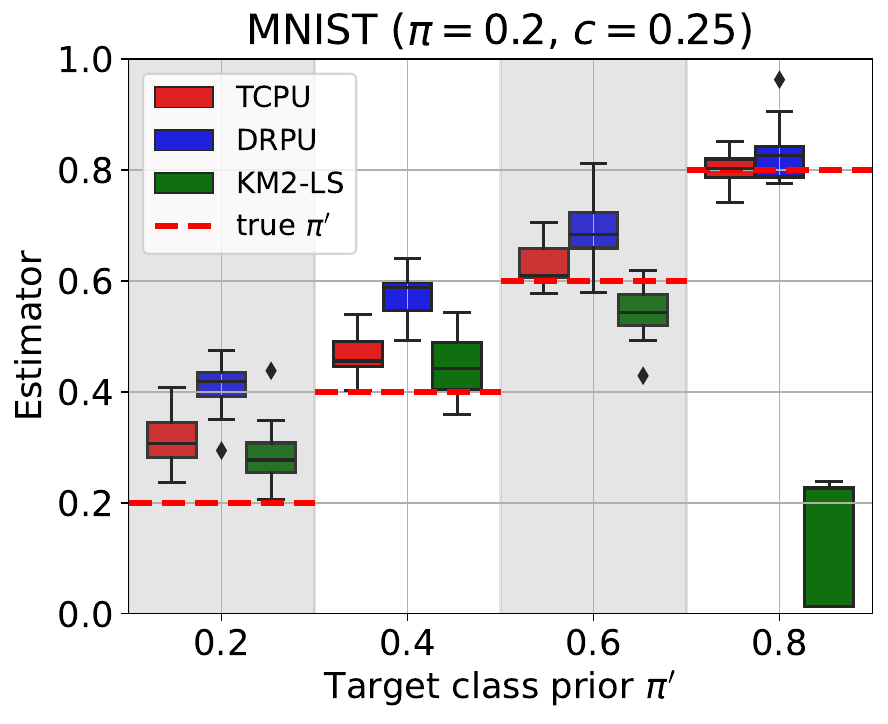}  &
\includegraphics[width=0.3\textwidth]{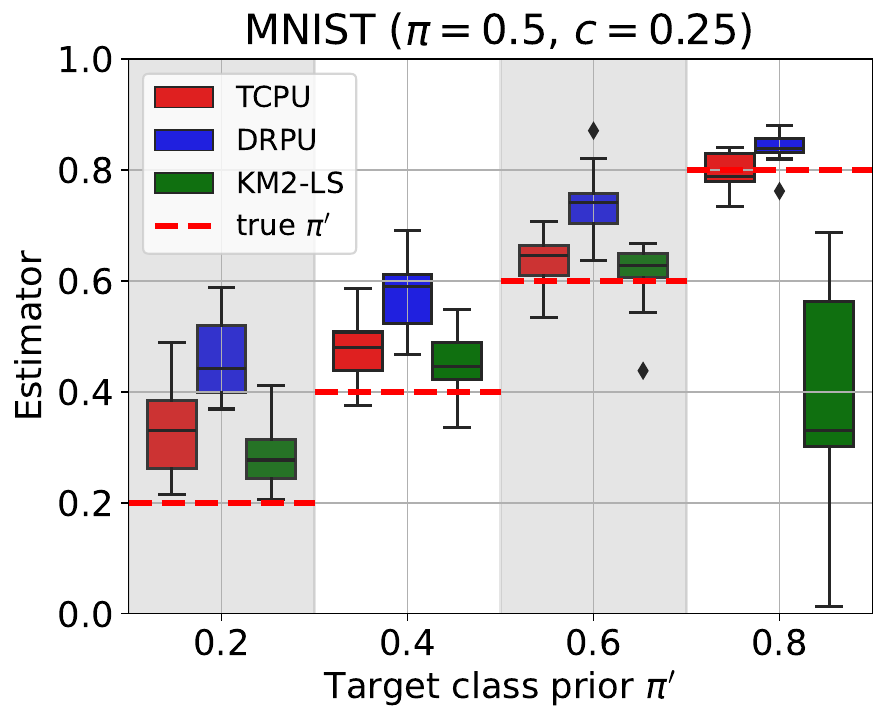} 
      \end{tabular}
    \caption{Distribution of estimators (red line indicates  the true $\pi'$) for image datasets for $c=0.25$.}
    \label{boxplot_image_main}
\end{figure*}

\subsection{Discussion}

In the main part of the paper we present selected results, while the full results can be found in the supplement.

Figures \ref{boxplot_artificial_1_main}-\ref{boxplot_image_main} in the main paper  show the distributions of estimators for synthetic and image datasets when $c=0.25$.
In the supplement, Figures \ref{boxplot_artificial_1}-\ref{boxplot_fashion} show the results for other parameter settings, whereas Tables \ref{tab:errors_image_c025}-\ref{tab:errors_uci2}  present the mean absolute value estimation errors and their standard deviations  for image and UCI datasets. Since our primary interest lies in analyzing estimation errors for small or moderate data samples, we exhibit the boxplots for the case where the source and target datasets consist of randomly chosen samples of 2000 observations. When the total number of observations in the original dataset is less than 4000, we split the data into source and target datasets in equal proportions.

Experiments indicate that the values of $\pi$ and $\pi'$ significantly influence the quality of the $\pi'$ estimation, with the impact differing across methods. 
The boxplots clearly indicate that KM2-LS significantly underestimates the true value of $\pi'$ when $\pi'$ is large. This effect becomes pronounced for low $c$ and small $\pi$ values, which is often the case in many practical applications. A similar problem with the KM2-LS method occurs for UCI datasets, where we observe very large estimation errors for $\pi=0.8$ (see Tables \ref{tab:errors_uci1} and \ref{tab:errors_uci2}).
Importantly, in the discussed situation, the proposed TCPU performs much better and allows for  more accurate estimation of $\pi'$, regardless of the values of $\pi$ and $c$.
The performance of DRPU depends on the particular dataset. For the  synthetic dataset as well as for the FASHION dataset, DRPU performs quite stably. However, for MNIST and CIFAR datasets, we see that DRPU significantly overestimates the true value of $\pi'$, especially for $\pi=0.1$ and $c=0.25$. In these cases, the proposed TCPU method returns estimators that are significantly closer to the theoretical values of $\pi'$.
Tables \ref{tab:errors_image_c025}-\ref{tab:errors_uci2}, containing estimation errors, reveal that the TCPU method returns the smallest estimation errors or errors that are not significantly different from the winning method for most datasets and parameter settings.

Since the proposed TCPU estimator requires $\pi$ as input, we also investigated the impact of the value $\pi$ on the quality of the TCPU estimation of $\pi'$. As expected (Figures \ref{boxplot_artificial_3_main} and \ref{boxplot_artificial_3}), the smallest errors are obtained when the TCPU estimator uses a known $\pi$ (this situation is marked with a blue dashed line). When the estimator $\hat{\pi}$ deviates from the true value of $\pi$, we obtain larger prediction errors for $\pi'$. Furthermore, we observe larger errors for $\pi=0.5$ than for small values of $\pi=0.1$ or $\pi=0.2$.

\begin{figure}[h!]
\centering
    \begin{tabular}{c}
      \includegraphics[width=0.4\textwidth]{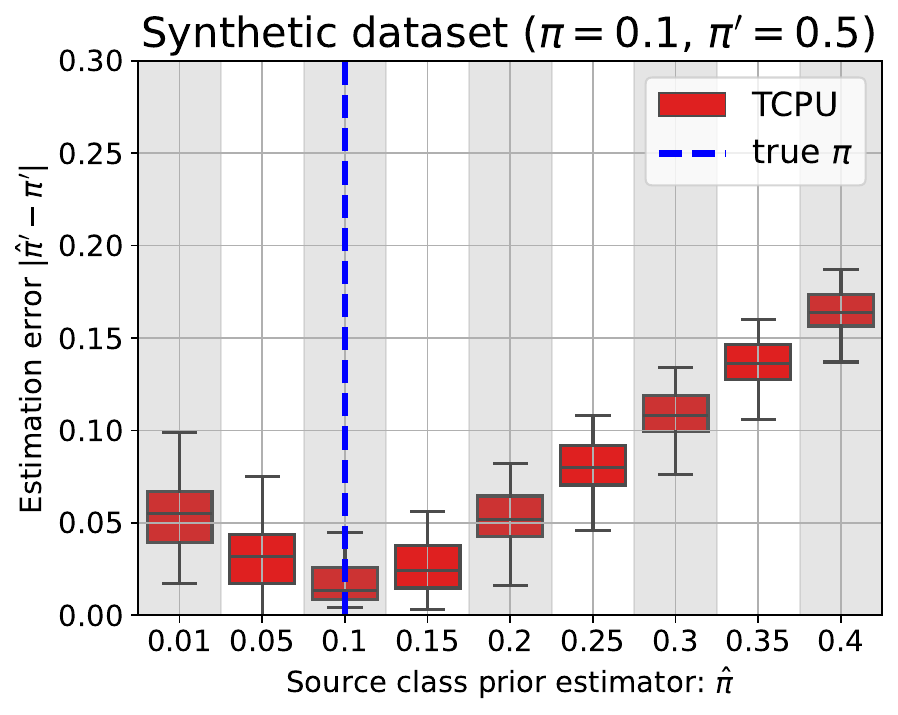} \\
       \includegraphics[width=0.4\textwidth]{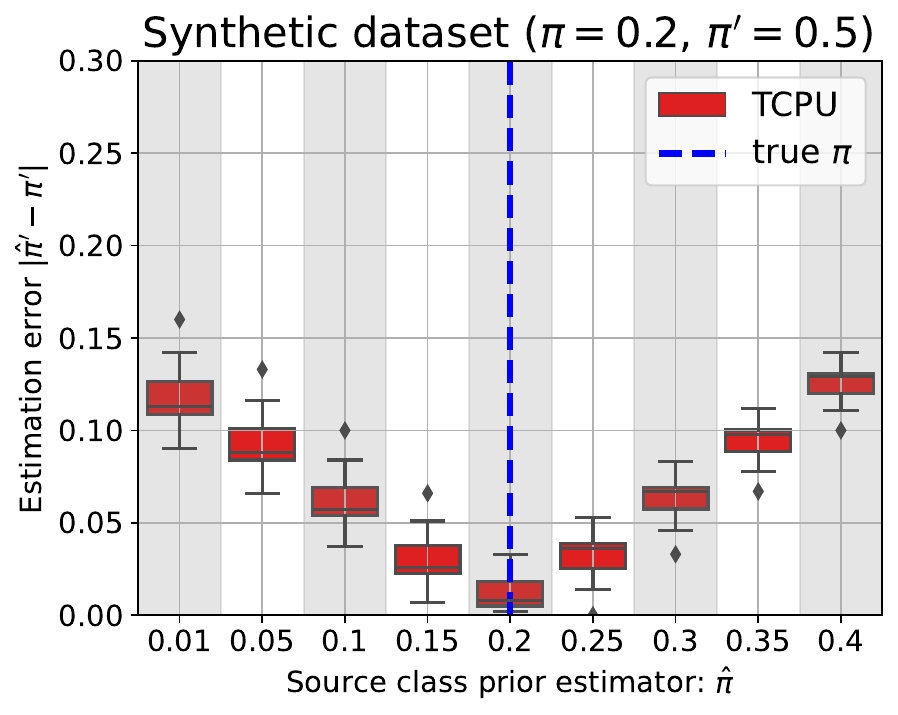}       
      \end{tabular}
    \caption{The impact of $\pi$ estimation on the performance of the TCPU estimator. The boxplot shows estimation errors for TCPU target class prior estimator $|\hpp-\pi'|$, for different source class prior estimators $\hat{\pi}$.}
    \label{boxplot_artificial_3_main}
\end{figure}

\begin{figure}[h!]
\centering
    \begin{tabular}{c}
      \includegraphics[width=0.4\textwidth]{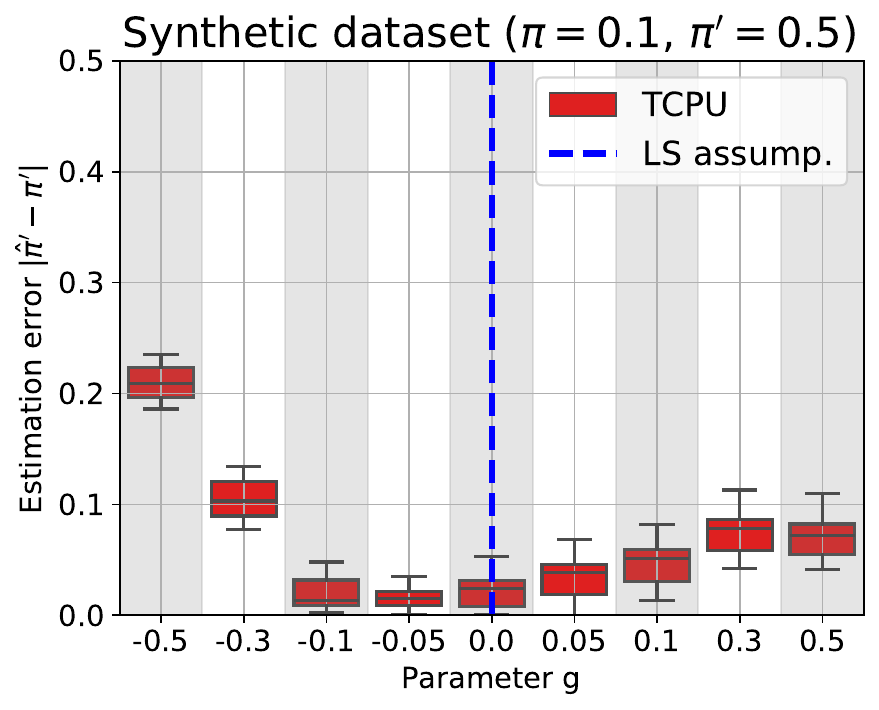} \\
       \includegraphics[width=0.4\textwidth]{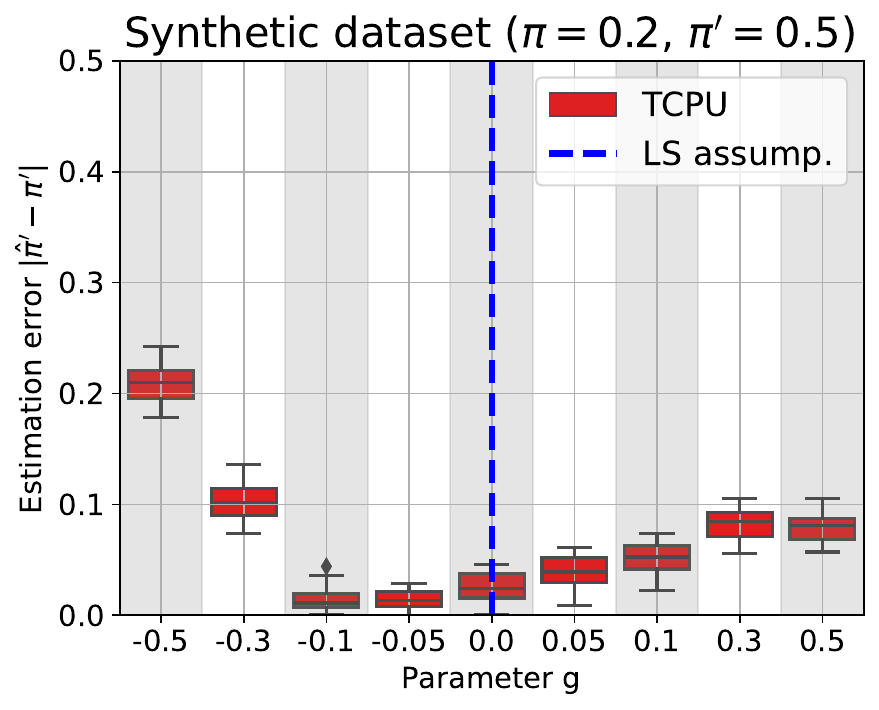}        
      \end{tabular}
    \caption{Robustness to violation of the Label Shift (LS) assumption (\ref{LSassumption}). The blue vertical line corresponds to the situation when the assumption is met, and non-zero values of the parameter $g$ indicate a violation of the assumption.}
    \label{boxplot_artificial_4_main}
\end{figure}

We also analyzed the TCPU method's robustness to violations of the label shift assumption (\ref{LSassumption}), indicating that the feature distribution for a given value of the target variable is the same in both the source and target datasets.
To this end, we modified the synthetic data generation method as follows.
As in previous experiments, negative observations are generated from a 10-dimensional normal distribution $N(0, I)$.
However, we change the way positive observations are generated. In the source dataset, we assume that positive observations are generated from $N(a, I)$, where $a = (1, \ldots, 1)$, whereas in the target dataset, the positive observations are generated from $N(a+g, I)$, where $g$ is the disturbance parameter. The value $g=0$ means that the assumption of invariance of the feature distribution within classes is met, while values other than zero indicate that it is violated. Figure \ref{boxplot_artificial_4_main} in the main paper and Figure \ref{boxplot_artificial_4} in the supplement show the distributions of TCPU estimation errors for different values of the parameter $g$.
As expected, the smallest errors usually occur for $g=0$. It can be seen that for small values of the disturbance parameter $g$, the method is quite robust and returns small errors. Larger disturbances, in line with intuition, lead to a deterioration in the method's performance. A more pronounced deterioration is observed for negative values of $g$ than for  positive ones.

The effect of dataset size on the quality of the estimation is shown in Figure \ref{lineplot_artificial_n} in the supplement. DRPU returns values close to zero for small sample sizes, which results in large estimation errors, while for larger sample sizes it performs comparable or only slightly worse than TCPU. Estimation errors for the KM2-LS method decrease quite slowly with the increase in the sample size. The advantage of TCPU over DRPU and KM2-LS is clear for small or moderate data sizes.

We analyzed the impact of the Gaussian kernel parameter $\tau$ on the estimation quality. The value $\tau=1/p$ corresponds to the default setting in the function \texttt{pairwise$\_$kernels} from the popular \texttt{scikit-learn} library in Python. This choice turned out to be quite reasonable in our experiments.
For example, for the CIFAR data, with $\pi = 0.1$ and $\pi' = 0.8$, the averaged estimation errors $|\pi' - \widehat{\pi}'|$ were $0.04, 0.02, 0.07$ for $\tau = 0.5/p, 1/p, 2/p$, respectively.
We conjecture that there is room for further improvement. For instance, if one could establish the asymptotic variance of TCPU, a natural approach would be to choose $\tau$ that minimizes the estimated asymptotic variance. Minimisation can be performed over quantiles of   empirical distribution of $\{||X_i-X_j||^{-2}, i\neq j\}$, see \cite{Vaz2019}.


Finally, we also analyzed the computation times for the individual methods (Table \ref{tab:times} in the supplement).
DRPU is the only method considered that requires model training on the source data, which usually leads to longer computation times.
Computing our TCPU estimator requires calculation of the kernel matrix for all pairs of observations, and this is the dominant cost of the entire algorithm. The cost of this step grows quadratically with the number of observations in the data. For example, for the synthetic dataset considered in the paper, the computation times were: 2.81, 5.97, 83.58, 331.56 and 3277.25 seconds for sample sizes of 500, 1K, 2K, 5K, and 10K respectively. However, as the times reported in Table \ref{tab:times} show, our method is still faster than the competing methods (DRPU and KM2-LS). Naturally, speed-up can be achieved using multithreaded computation: computing the kernel matrix is easy to parallelize, and such an option is available in popular functions like pairwise kernels in the scikit-learn library. Moreover, in order to reduce quadratic computational costs of U-statistics which are needed to calculate TCPU estimator, one can use incomplete U-statistics based on random sample of pairs sampled from the set of all pairs \citep{Schrab22}. In the manuscript, we focused on smaller datasets (up to 2000 observations) because class-prior estimation is most demanding in that setting.

\section{CONCLUSIONS AND FUTURE WORK}
In this paper, we introduced a novel estimator, TCPU, for estimating the target class prior $\pi'$ under a label shift in the context of positive unlabeled (PU) learning.
Our approach, which leverages distribution matching and kernel embedding, provides an explicitly expressed estimate of the target class prior without estimating posterior probabilities via classifier training.
We proved  that the TCPU estimator is  asymptotically consistent and, in Theorem \ref{th_random},  established a calculable non-asymptotic error bound.
Experimental results on synthetic and real datasets show that TCPU outperforms existing methods, particularly in cases of small source class prior $\pi$ values. 
Importantly, TCPU avoids the problem of significant underestimation or overestimation of the target class prior, which is a drawback of both competing methods DRPU and KM2-LS.
A natural direction for future research is to further investigate the effect of $\pi$ and $\pi'$ estimation on classification accuracy and to develop  testing methods  for occurrence of  label shift what is of interest for the quantification problems.  Also,  generalisation of  the proposed method to not necessarily binary  nominal response  by considering the extension of a PU scenario to a noisy data model (\cite{Zhang2018}) is of interest.

\bibliography{References}

@inproceedings{BekkerAAAI18,
  author    = {Bekker, J. and Davis, J.},
  title     = {Estimating the Class Prior in Positive and Unlabeled Data through Decision Tree Induction},
  booktitle = {Proceedings of the 32th {AAAI} Conference on Artificial Intelligence},
pages = {1-8},
  year={2018}
}

@inproceedings{Schrab22,
author = {Schrab, A. and Kim, I. and Guedj, B. and Gretton, A.},
title = {Efficient aggregated kernel tests using incomplete U-statistics},
year = {2022},
booktitle = {Proceedings of the 36th International Conference on Neural Information Processing Systems},
series = {NIPS'22}
}

@article{Vaz2019,
  author  = {Vaz, A. and  Izbicki, R. and  Stern, R.},
  title   = {Quantification Under Prior Probability Shift: the Ratio Estimator and its Extensions},
  journal = {Journal of Machine Learning Research},
  year    = {2019},
  volume  = {20},
  pages   = {1--33}
 }

@article{Tols2017,
  author  = {Tolstikhin, I. and  Sriperumbudur, B. K. and Muandet, K.},
  title   = {Minimax Estimation of Kernel Mean Embeddings},
  journal = {Journal of Machine Learning Research},
  year    = {2017},
  volume  = {18},
  number  = {86},
  pages   = {1--47}
 
}

@article{Blanchard2010,
  author  = {Blanchard, G. and Lee, G. and Scott, C.},
  title   = {Semi-supervised novelty detection},
  journal = {Journal of Machine Learning Research},
  year    = {2010},
  volume  = {11},
  pages   = {2973--3009}
 }

@article{Forman2008,
  author  = { Forman, G.},
 title   = {Quantifying counts and costs via classification},
  journal = {Data Mining and Knowledge Discovery},
  year    = {2008},
  volume  = {17},
  pages   = {164--206}
 }

@article{Gretton2012,
author = {Gretton, A. and Borgwardt, K. and Rasch, M. and Sch\"olkopf, B. and Smola, A.},
title={A kernel two-sample test},
year = {2012},
volume = {13},
journal = {Journal of Machine Learning Research},
pages = {723-773},
}

@article{BekkerDavis2020,
	Author = {Bekker, J. and Davis, J.},
	Journal = {Machine Learning},
	Title = {Learning from positive and unlabeled data: a survey},
	pages={719-760},
	volume = {109},
	Year = {2020},
	}

@InProceedings{Zhao2022,
author = {Zhao, Y. and Xu, Q. and Jiang, Y. and Wen, P. and Huang, Q.}, 
title = {Dist-PU: Positive-Unlabeled Learning From a Label Distribution Perspective}, booktitle = {Proceedings of the Conference on Computer Vision and Pattern Recognition},
series = {CVPR'22},
year = {2022}, 
pages = {14461-14470} 
}

@article{Gonzales2017,
author = {Gonz\'{a}lez, P. and Casta\~{n}o, A. and Chawla, N.  and Coz, J.},
title = {A Review on Quantification Learning},
year = {2017},
journal = {ACM Comput. Surv.},
volume = {50},
number = {5}
}

@inproceedings{SelfPU,
author = {Chen, X. and Chen, W. and Chen, T. and Yuan, Y. and Gong, C. and Chen, K. and Wang, Z.},
title = {{Self-PU: Self boosted and calibrated positive-unlabeled training}},
year = {2020},
booktitle = {Proceedings of the 37th International Conference on Machine Learning},
articleno = {141},
numpages = {10},
series = {ICML'20}
}

@inproceedings{Luo2021_PULNS, 
title={PULNS: Positive-Unlabeled Learning with Effective Negative Sample Selector}, 
volume={35}, 
author={Luo, C. and Zhao, P. and Chen, C. and Qiao, B. and Du, C.  and Zhang, H. and Wu, W. and Cai, S. and He, B. and Rajmohan, S. and Lin, Q.},
booktitle            ={Proceedings of the AAAI Conference on Artificial Intelligence},
year={2021},
series = {AAAI'21},
pages={8784-8792} 
}

@inproceedings{Zhang2013,
  author={Zhang, K. and Sch\"olkopf, B. and Muandet, K. and Wang, Z.},
  booktitle={Proceedings of the 30th International Conferencce on Machine Learning }, 
  title={Domain adaptation under target and conditional shift}, 
 year={2014}
}

@inproceedings{Fukimizu2007,
 author = {Fukumizu, K. and Gretton, A. and  Sun, X. and Sch\"{o}lkopf, B.},
 booktitle = {Advances in Neural Information Processing Systems},
 title = {Kernel Measures of Conditional Dependence},
volume = {20},
 year = {2007}
}

@inproceedings{Iyer2014,
  author={Iyer, A. and Nath, S. and  Sarawagi, S.},
  booktitle={Proceedings of the 31th International Conferencce on Machine Learning }, 
  title={Maximum mean discrepancy for class ratio estimation: convergence bounds and kernel selection}, 
  year={2014},
 series = {IMLR W \& CP vol. 32}
 }

@inproceedings{Dussap2023,
  author={Dussap, B. and Blanchard, G. and  Ch\'erif-Abdellatif, B.-E.},
  booktitle={Proceedings of the European  Conferencce on Machine Learning }, 
  title={ Label shift quantification  with robustness guarantees  via distribution feature matching }, 
  year={2023},
}

@inproceedings{Kiryo2017,
 author = {Kiryo, R. and Niu, G. and du Plessis, M. C. and Sugiyama, M.},
 title = {Positive-unlabeled Learning with Non-negative Risk Estimator},
 booktitle = {Proceedings of the International Conference on Neural Information Processing Systems},
 series = {NIPS'17},
 year = {2017},
 pages = {1674--1684}
}

@inproceedings{Jainetal2016, 
author = {Jain, S. and White, M. and Radivojac, P.},
 title = {Estimating the Class Prior and Posterior from Noisy Positives and Unlabeled Data}, 
 year = {2016}, 
 booktitle = {Proceedings of the 30th International Conference on Neural Information Processing Systems}, 
 pages = {2693–2701}
}

@inproceedings{ElkanNoto2008,
 author = {Elkan, C. and Noto, K.},
 title = {Learning Classifiers from Only Positive and Unlabeled Data},
 booktitle = {Proceedings of the 14th ACM SIGKDD International Conference on Knowledge Discovery and Data Mining},
 series = {KDD '08},
 year = {2008},
 pages = {213--220}
}

@inproceedings{LiptonWangSmola2018,
  author       = {Lipton, Z. C. and Wang, Y. and Smola, A. J.},
  title        = {Detecting and Correcting for Label Shift with Black Box Predictors},
  booktitle    = {Proceedings of the 35th International Conference on Machine Learning},
  series       = {ICML' 18},
  pages        = {3128-3136},
  year         = {2018},
}

@inproceedings{Gargetal2020,
author = {Garg, S. and Wu, Y. and Balakrishnan, S. and Lipton, Z. C.},
title = {A unified view of label shift estimation},
year = {2020},
booktitle = {Proceedings of the 34th International Conference on Neural Information Processing Systems},
series = {NIPS' 20},
pages = {1-11},
}

@article{SaerensLatinneDecaestecker2002,
author = {Saerens, M. and Latinne, P. and Decaestecker, C.},
title = {Adjusting the outputs of a classifier to new a priori probabilities: a simple procedure},
year = {2002},
volume = {14},
number = {1},
journal = {Neural Comput.},
pages = {21–41},
}

@article{Nakajima23,
    title = { Positive-unlabeled classification under class-prior shift: a prior-invariant approach based on density ratio estimation},
    journal = {Machine Learning},
    volume = {112},
    pages = {889-919},
    year = {2023},
       author = {Nakajima, S.  and  Siguyama, M. },
}

@article{McDiarmid1989,
    title = { On the method of bounded differences},
    journal = {Survey in Combinatorics},
    pages = {148-188},
    year = {1989},
       author = {Mc Diarmid, C. }
}

@article{Rolandaetal2022,
  author={Roland, T. and  Bock, C. and  Tschoellitsch, T. and Maletzky, A. and Hochreiter, S. and Meier, J. and Klambauer, G.},
  journal={Journal of Medical Systems}, 
  title={Domain Shifts in Machine Learning Based Covid-19 Diagnosis From Blood Tests}, 
  year={2022},
  volume={46},
  number={5},
  pages={1--12}
}

@InProceedings{Ramaswamy2016,
  title = 	 {Mixture Proportion Estimation via Kernel Embeddings of Distributions},
  author = 	 { Ramaswamy, H.  and   Scott, C.  and Tewari, A.},
  booktitle = 	 {Proceedings of The 33rd International Conference on Machine Learning},
  pages = 	 {2052--2060},
  year = 	 {2016},
  volume = 	 {48}
}

@article{Fung2006,
 author = {Fung, G. P. C. and Yu, J. X. and Lu, H. and Yu, P. S.},
  title = {Text Classification without Negative Examples Revisit}, 
  year = {2006},
    volume = {18}, 
    number = {1},
     journal = {IEEE Transactions on Knowledge and Data Engineering}, 
      pages = {6–20}  
}

@article{Gong2021,
author = {Gong, C. and Wang, Q. and  Liu, T. and Han, B. and You, J. and Yang, J. and Tao, D.},
title = {Instance-Dependent Positive and Unlabeled Learning with Labeling Bias Estimation},
journal = {IEEE Trans Pattern Anal Mach Intell},
volume = {},
pages = {1--16},
year = {2021}
}

@article{Li2021,
    author = {Li, F. and Dong, S. and Leier, A. and Han, M. and Guo, X. and Xu, J. and Wang, X. and Pan, S. and Jia, C. and Zhang, Y. and Webb, G.  and Coin, L. J. M. and Li, C. and Song, J.},
    title = {Positive-unlabeled learning in bioinformatics and computational biology: a brief review},
    journal = {Briefings in Bioinformatics},
    volume = {23},
    number = {1},
    year = {2021}
}

@inproceedings{LiLiu2003,
 author = {Li, X. and Liu, B.}, 
 title = {Learning to Classify Texts Using Positive and Unlabeled Data}, 
 year = {2003},
 booktitle = {Proceedings of the 18th International Joint Conference on Artificial Intelligence},
 series = {IJCAI'03},
 pages = {587–592}
  }

@inproceedings{PyTorch19,
title = {PyTorch: An Imperative Style, High-Performance Deep Learning Library},
author = {Paszke, A. and Gross, S. and Massa, F. and Lerer, A. and Bradbury, J. and Chanan, G. and Killeen, T. and Lin, Z. and Gimelshein, N. and Antiga, L. and Desmaison, A. and Kopf, A. and Yang, E. and DeVito, Z. and Raison, M. and Tejani, A. and Chilamkurthy, S. and Steiner, B. and Fang, L. and Bai, J. and Chintala, S.},
booktitle = {Advances in Neural Information Processing Systems},
pages = {8024--8035},
series  ={NIPS'19},
year = {2019},
}

@ARTICLE{MielniczukWawrzenczyk2024,
  author={Mielniczuk, J. and Wawrzenczyk, A.},
  journal={Fundamenta Informaticae}, 
  title={{Single-sample versus case-control sampling scheme for Positive
Unlabeled data: the story of two scenarios}}, 
  year={2024},
  volume={191},
issue = {2},
  pages={1--17},
 }

@inproceedings{Na2020,
author = {Na, B. and Kim, H. and Song, K. and Joo, W. and Kim, Y. and Moon, I.},
title = {Deep Generative Positive-Unlabeled Learning under Selection Bias},
year = {2020},
booktitle = {Proceedings of the 29th ACM International Conference on Information \& Knowledge Management},
pages = {1155–1164},
series = {CIKM '20}
}

@inproceedings{Sakai2019,
author={Sakai, T. and Shimizu, N.},
title={Covariate Shift Adaptation on Learning from Positive and Unlabeled Data},
year={2019},
booktitle={Proceedings of the AAAI Conference on Artificial Intelligence},
pages={4838-4845}
}

@inproceedings{Hammoudeh2020,
author = {Hammoudeh, Z. and Lowd, D.},
title = {Learning from positive and unlabeled data with arbitrary positive shift},
year = {2020},
booktitle={Proceedings of the 34th International Conference on Neural Information Processing Systems}
}

@inproceedings{Zhang2018,
    title={Generalizec cross entropy loss for training neural networks with noisy labels},
    author={Zhang, Z. and M. Sabuncu},
    booktitle={NIPS'18},
     pages={8792 - 8802},
    year={2018}
}

@article{Sechidis2017,
title = {Dealing with under-reported variables: An information theoretic solution},
journal = {International Journal of Approximate Reasoning},
volume = {85},
pages = {159 - 177},
year = {2017},
author = {Sechidis, K. and Sperrin, M. and Petherick, E. S. and  Luján, M. and  Brown, G.}
}





\clearpage

\section*{Checklist}



\begin{enumerate}

  \item For all models and algorithms presented, check if you include:
  \begin{enumerate}
    \item A clear description of the mathematical setting, assumptions, algorithm, and/or model. [Yes]
    \item An analysis of the properties and complexity (time, space, sample size) of any algorithm. [Yes]
    \item (Optional) Anonymized source code, with specification of all dependencies, including external libraries. [Yes]
  \end{enumerate}

  \item For any theoretical claim, check if you include:
  \begin{enumerate}
    \item Statements of the full set of assumptions of all theoretical results. [Yes]
    \item Complete proofs of all theoretical results. [Yes]
    \item Clear explanations of any assumptions. [Yes]     
  \end{enumerate}

  \item For all figures and tables that present empirical results, check if you include:
  \begin{enumerate}
    \item The code, data, and instructions needed to reproduce the main experimental results (either in the supplemental material or as a URL). [Yes]
    \item All the training details (e.g., data splits, hyperparameters, how they were chosen). [Yes]
    \item A clear definition of the specific measure or statistics and error bars (e.g., with respect to the random seed after running experiments multiple times). [Yes]
    \item A description of the computing infrastructure used. (e.g., type of GPUs, internal cluster, or cloud provider). [Yes]
  \end{enumerate}

  \item If you are using existing assets (e.g., code, data, models) or curating/releasing new assets, check if you include:
  \begin{enumerate}
    \item Citations of the creator If your work uses existing assets. [Not Applicable]
    \item The license information of the assets, if applicable. [Not Applicable]
    \item New assets either in the supplemental material or as a URL, if applicable. [Not Applicable]
    \item Information about consent from data providers/curators. [Not Applicable]
    \item Discussion of sensible content if applicable, e.g., personally identifiable information or offensive content. [Not Applicable]
  \end{enumerate}

  \item If you used crowdsourcing or conducted research with human subjects, check if you include:
  \begin{enumerate}
    \item The full text of instructions given to participants and screenshots. [Not Applicable]
    \item Descriptions of potential participant risks, with links to Institutional Review Board (IRB) approvals if applicable. [Not Applicable]
    \item The estimated hourly wage paid to participants and the total amount spent on participant compensation. [Not Applicable]
  \end{enumerate}

\end{enumerate}

\appendix
\onecolumn
\aistatstitle{Prior shift estimation for positive unlabeled data through the lens of kernel embedding: \\
Supplementary Materials}

\section{Proofs}

\subsection{Proof of Theorem \ref{th_cons}}


The proof of (i) follows from  Chebyshev's inequality applied for \[\Phi(\hat P)=n^{-1}\sum_{i=1}^n \Phi(X_i)\] as well as $\Phi(\hat P_+)$ and $\Phi(\hat P')$: fix $\varepsilon > 0$, then
\begin{equation}
\label{form1}
P(||\Phi(\hat P)- \Phi(P)||_{\cal H}>\varepsilon)\leq \frac{\mathbbm{E}||\Phi(\hat P)  - \Phi(P)||^2_{\cal H}}{\varepsilon ^2}.
\end{equation} 
Now we focus on bounding the numerator in \eqref{form1}. First, we obtain
\begin{equation}
    \label{form2}
\mathbbm{E}||\Phi(\hat P)  - \Phi(P)||^2 = \mathbbm{E}|| \Phi(\hat P)||^2_{\cal H} -
2 \mathbbm{E} <\Phi(\hat P), \Phi( P)>_{\cal H} + || \Phi( P)||^2_{\cal H}.
\end{equation}
Next, we consider the two first terms on the right-hand side of \eqref{form2}. For the first one, we have:
\begin{eqnarray*}
\mathbbm{E} \left|\left|\frac{1}{n}\sum_{i=1}^n \phi(X_i)\right|\right|^2_{\cal H} &=& \frac{1}{n^2} \sum_{i=1}^n \sum_{j=1}^n <\phi(X_i),\phi(X_j)>_{\cal H} 
=\frac{1}{n^2} \sum_{i=1}^n \sum_{j=1}^n K(X_i,X_j)\\
&=& \frac{1}{n^2} \sum_{i=1}^n \mathbbm{E} K(X_i,X_i) + \frac{1}{n^2} \sum_{1 \leq i\neq j \leq n} \mathbbm{E} K(X_i,X_j)\\
&=& \frac{1}{n} \mathbbm{E} K(X_1,X_1) + \left(1-\frac{1}{n}\right) ||\Phi(P)||^2_{\cal H}.
\end{eqnarray*}
Moreover, we have 
\begin{equation}
\label{form3}
    \mathbbm{E} <\Phi(\hat P), \Phi( P)>_{\cal H} = \frac{1}{n} \sum_{i=1}^n \mathbbm{E} < \phi(X_i), \Phi(P)>_{\cal H} = \mathbbm{E} < \phi(X_1), \Phi(P)>_{\cal H}.
\end{equation}
From the reproducing property we obtain $<\phi(x_1) , \Phi(P)>_{\cal H} = [\Phi(P)](x_1) = \mathbbm{E} K(X,x_1).$ Therefore, the right-hand side of \eqref{form3} equals $||\Phi(P)||^2_{\cal H}.$ Finally, \eqref{form2} is $\frac{1}{n} [\mathbbm{E} K(X,X) - ||\Phi(P)||^2_{\cal H}],$ so the right-hand side of \eqref{form1} tends to zero as $n$ goes to infinity. This fact implies that $\Phi(\hat P)$ tends to $\Phi( P)$
in probability. Obviously, the analogous properties hold for 
$\Phi(\hat P _ +)$ and $\Phi(\hat P ')$. In particular,  we have that $||\Phi(\hat P) -\Phi(\hat P_+)||_{\cal H}$ is positive with probability tending to one as $N \to \infty$. It follows from continuity of a norm
and the assumptions that
$$
||\Phi(\hat P) -\Phi(\hat P_+)||_{\cal H} \to_P 
||\Phi( P) -\Phi( P_+)||_{\cal H}=
(1-\pi) ||\Phi( P_ -) -\Phi( P_+)||_{\cal H}.
$$

Thus, the second equality in Lemma 2 (main paper) and continuity of a scalar product give
\[\hat\pi' \to_P \, 1-\frac{(1-\pi)<\Phi( P) -\Phi( P_+),\Phi({P'})- \Phi({P_+})>_{\cal H}}{||\Phi( P) -\Phi( {P_+})||_{\cal H}^2} =\pi' .\]

Next, we focus on the proof of (ii). 
From the first equality in (\ref{explicit}) and the Cauchy–Schwarz inequality we have
\begin{align}
    \label{error_term}
|\hpp - \pi '| & = \frac{<\Phi(\hat P) -\Phi(\hat P_+), \Delta - \pi ' [\Phi(\hat P) -\Phi(\hat P_+)]>_{\cal H}}{||\Phi(\hat P) -\Phi(\hat P_+)||_{\cal H}^2} \\
&\leq 
\frac{||\Delta - \pi ' [\Phi(\hat P) -\Phi(\hat P_+)]||_{\cal H}}{||\Phi(\hat P) -\Phi(\hat P_+)||_{\cal H}} \nonumber.
\end{align}
Notice that
$$
\Delta - \pi ' [\Phi(\hat P) -\Phi(\hat P_+)] = (1-\pi ') \Phi(\hat P) - (1-\pi) \Phi(\hat P ') + (\pi ' -\pi) \Phi(\hat P_+)
$$
and 
$$
(1-\pi ') \Phi( P) - (1-\pi) \Phi( P ') + (\pi ' -\pi) \Phi( P_+) = 0,
$$
which imply that 
\begin{align*}
&||\Delta - \pi ' [\Phi(\hat P) -\Phi(\hat P_+)]||_{\cal H}  \leq 
(1-\pi ') ||\Phi(\hat P) - \Phi( P)||_{\cal H}   \\
&+ (1-\pi) || \Phi(\hat P ') - \Phi( P ')||_{\cal H} + |\pi ' -\pi| \times ||\Phi(\hat P_+) - \Phi( P_+)||_{\cal H}.
\end{align*}
Now we use Lemma \ref{Lemma3} (given below) and consider the event on which all three inequalities in this lemma hold. 
In this case 
\begin{equation}
    \label{fff}
||\Delta - \pi ' [\Phi(\hat P) -\Phi(\hat P_+)]||_{\cal H} \leq
4 \sqrt{\frac{M}{N} \log(1/\delta)} 
\end{equation}
and 
$$
||\Phi(\hat P) -\Phi(\hat P_+)||_{\cal H} \geq 
||\Phi( P) -\Phi( P_+)||_{\cal H} - ||\Phi(\hat P) - \Phi( P)||_{\cal H} -||\Phi(\hat P_+) - \Phi( P_+)||_{\cal H},
$$
which again can be bounded by 
$$
(1-\pi)||\Phi( P_-) - \Phi( P_+)||_{\cal H} - 4 \sqrt{\frac{M}{N} \log(1/\delta)}.
$$
From the assumption \eqref{assum} the above expression is not smaller than $\alpha (1-\pi)||\Phi( P_-) - \Phi( P_+)||_{\cal H}.$ In particular, we have that  $||\Phi(\hat P) -\Phi(\hat P_+)||_{\cal H}$ is positive with high probability. Since the numerator and the denominator on the right-hand side of \eqref{error_term} are bounded, the proof of (ii) is finished.

\subsection{Proof of Theorem \ref{th_random}}
The proof is even simpler than the proof of Theorem 1 and involves bounding the numerator on the right-hand side of \eqref{error_term} done in \eqref{fff}.

\subsection{Lemma 3 and its proof}

\begin{lemma}
\label{Lemma3}    
Suppose that $M=\sup_x K(x,x)<\infty$ and $\delta \leq \exp(-(\sqrt{2}+1)^2/2) $ is arbitrary. Then with probability at least $1-3\delta$ the following three inequalities simultaneously hold
\begin{align*}
& ||\Phi(\hat P) - \Phi( P)||_{\cal H} \leq 2 \sqrt{\frac{M}{n} \log(1/\delta)}\:, &
||\Phi(\hat P') - \Phi( P')||_{\cal H} \leq 2 \sqrt{\frac{M}{n'} \log(1/\delta)} \:,\\
&||\Phi(\hat P_+) - \Phi( P_+)||_{\cal H} \leq 2 \sqrt{\frac{M}{m} \log(1/\delta)} \:. & \\ 
\end{align*}
\end{lemma}

\begin{proof}
  Fix $\delta \leq \exp(-(\sqrt{2}+1)^2/2).$ We focus on the first inequality in the lemma. The initial step is McDiarmid's inequality \citep{McDiarmid1989} applied for the function
  $$
F(X_1,\ldots, X_n) = ||\Phi(\hat P) - \Phi( P)||_{\cal H}=
\left|\left|\frac{1}{n}\sum_{i=1}^n \phi(X_i) - \Phi( P)\right|\right|_{\cal H}.
  $$
We are to bound a difference
$$
| F(X_1,\ldots, X_n) - F(X_1,\ldots, X_i',\ldots , X_n)| \leq \frac{1}{n} ||\phi(X_i) - \phi( X_i')||_{\cal H} \leq \frac{2\sqrt{M}}{n}\;.
$$
Therefore, we obtain that with probability at least $1-\delta$
$$
||\Phi(\hat P) - \Phi( P)||_{\cal H} \leq
\mathbbm{E}||\Phi(\hat P) - \Phi( P)||_{\cal H} + \sqrt{\frac{2M \log(1/\delta)}{n}}.
$$
Applying Jensen's inequality to the  expectation above, we get  
$$
\mathbbm{E}||\Phi(\hat P) - \Phi( P)||_{\cal H} \leq \sqrt{\mathbbm{E}||\Phi(\hat P) - \Phi( P)||_{\cal H}^2 }.
$$
The expression under the square root has been already considered in the Proof of Theorem \ref{th_cons}  and was bounded  there by $\frac{1}{n} [\mathbbm{E} K(X,X) - ||\Phi(P)||^2_{\cal H}],$ which is smaller than $\frac{M}{n}\:.$ Note that for the chosen $\delta$ we have
$1+\sqrt{2\log(1/\delta)} \leq 2 \sqrt{\log(1/\delta)},$ which finishes the proof.
\end{proof}
The similar concentration inequalities are proven, for instance, in  \cite[Proposition A.1 and Remark~A.2]{Tols2017}.



\clearpage

\section{Computational times}
All experiments are executed in an isolated manner on the
same machine with access to 32Gb RAM and 8 cores
of an  Intel(R) Xeon(R) CPU E31270 3.40GHz.
\begin{table}[ht!]
\begin{center}
\caption{{\bf Computational times (in seconds)}, for example parameter setting $\pi=0.2$ and $\pi'=0.6$.}
\label{tab:times}
\begin{tabular}{l|lll}
\toprule 
Dataset &	TCPU  &	DRPU   &	KM2-LS \\
\midrule
Synthetic &	83.582 $\pm$ 0.473 &	\bf  48.72 $\pm$ 0.213 &	162.818 $\pm$ 1.855 \\
\midrule
MNIST &	\bf  155.415 $\pm$ 2.28 &	440.129 $\pm$ 18.567 &	280.33 $\pm$ 5.337 \\
CIFAR &	\bf  78.873 $\pm$ 0.997 &	151.749 $\pm$ 5.435 &	160.437 $\pm$ 1.658 \\
Fashion &	\bf  38.371 $\pm$ 0.319 &	48.768 $\pm$ 0.291 &	80.313 $\pm$ 0.384 \\
\midrule
Diabetes &	\bf  0.043 $\pm$ 0.009 &	0.293 $\pm$ 0.014 &	0.173 $\pm$ 0.019 \\
Spambase &	\bf  0.015 $\pm$ 0.004 &	0.023 $\pm$ 0.005 &	0.102 $\pm$ 0.018 \\
Segment &	0.052 $\pm$ 0.012 &	\bf  0.039 $\pm$ 0.015 &	0.194 $\pm$ 0.033 \\
Waveform &	0.07 $\pm$ 0.004 &	0.087 $\pm$ 0.009 &	\bf  0.036 $\pm$ 0.006 \\
Vehicle &	\bf  0.073 $\pm$ 0.015 &	0.091 $\pm$ 0.015 &	0.318 $\pm$ 0.016 \\
Yeast &	\bf  0.117 $\pm$ 0.016 &	0.22 $\pm$ 0.02 &	0.249 $\pm$ 0.018 \\
Banknote &	0.041 $\pm$ 0.011 &	\bf  0.026 $\pm$ 0.006 &	0.16 $\pm$ 0.033 \\
\bottomrule
\end{tabular}
\end{center}
\end{table}

\section{Technical details about the DRPU method}

The DRPU method  requires learning a parametric model. As in the original work, we used
5-layer Multi-Layer Perceptron  (MLP) : $p-300-300-300-1$ (with $p$ being the size of the feature vector) with ReLU activation, and trained by Adam with the default momentum parameters $\beta_1 = 0.9$ and $\beta_2 = 0.999$ and 
 $\ell_2$ regularization parameter
$5 \times 10^{3}$. Training was performed for $1000$ epochs with the batch size $100$ and learning rate $2\times 10^{-5}$.
The DRPU method requires setting the hyperparameter $\alpha$ for non-negative correction. In the experiments, for image datasets we used the values proposed in \cite{Nakajima23}: $\alpha=0.475, 0.425, 0.6$, for MNIST, CIFAR10 and FashionMNIST, respectively. For the remaining datasets, we set the value $\alpha=\hat{\pi}$, where $\hat{\pi}$ is the KM estimator, using the fact that $\alpha$ should satisfy $0\leq \alpha\leq \pi$. This is usually a sensible choice, for image datasets there is usually no significant difference between $\alpha=\hat{\pi}$ and the default values given above. 

\section{Controlling the size of the source data and the labeling frequency.}
We follow the procedure described e.g. in \cite{MielniczukWawrzenczyk2024} 
to control the size of the source data and the labeling frequency $c$, which represents the percentage of labeled observations among all positive observations.
Let $n$ denote the total number of observations in the source dataset. We pick $c\times\pi\times n$ observations from positive class and $(1-c)\times n$ from the whole data set. Thus $c\times \pi\times n+\pi\times (1-c)\times n = \pi\times n$ is an expected number of observations from the positive class in the sample and fraction $c$ of them will be labeled on average. In order to ensure that the chosen sample has size equal to $n$, both sizes should be increased $A = (1-c(1-\pi))^{-1}$ times i.e. size of chosen labeled sample should be $A \times c \times \pi \times n$ and $A \times (1 - c) \times n$ for the unlabeled one.
Note that both samples are not necessarily disjoint.

\clearpage

\begin{figure}[ht!]
\centering
    \begin{tabular}{c c}
    \includegraphics[width=0.45\textwidth]{figures1/exp1_est_art1_c0.25_n_sample2000_pi_train0.1.pdf}  &
      \includegraphics[width=0.45\textwidth]{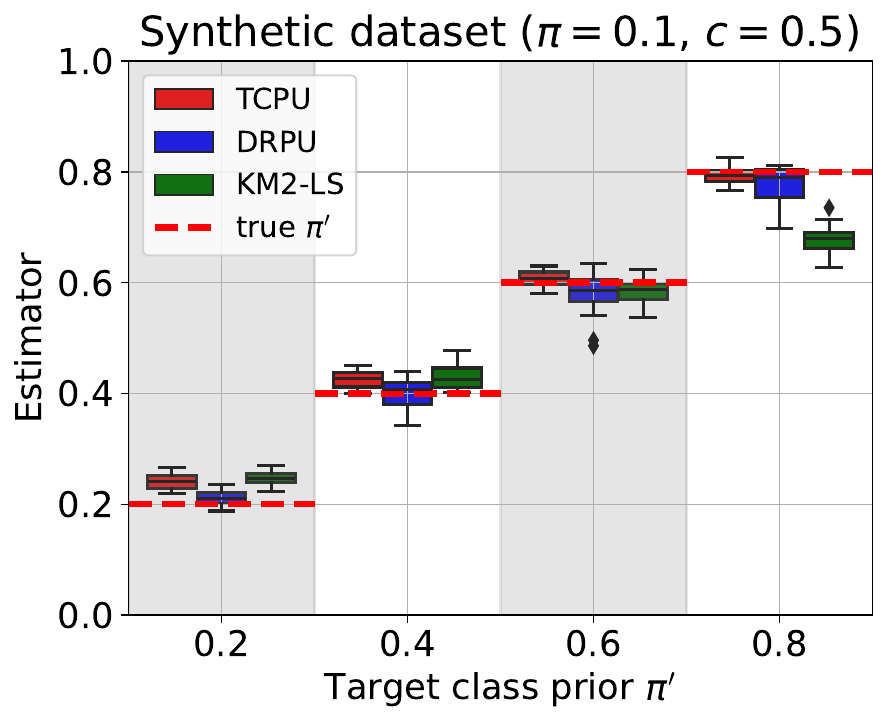}  \\
          \includegraphics[width=0.45\textwidth]{figures1/exp1_est_art1_c0.25_n_sample2000_pi_train0.2.pdf}&
      \includegraphics[width=0.45\textwidth]{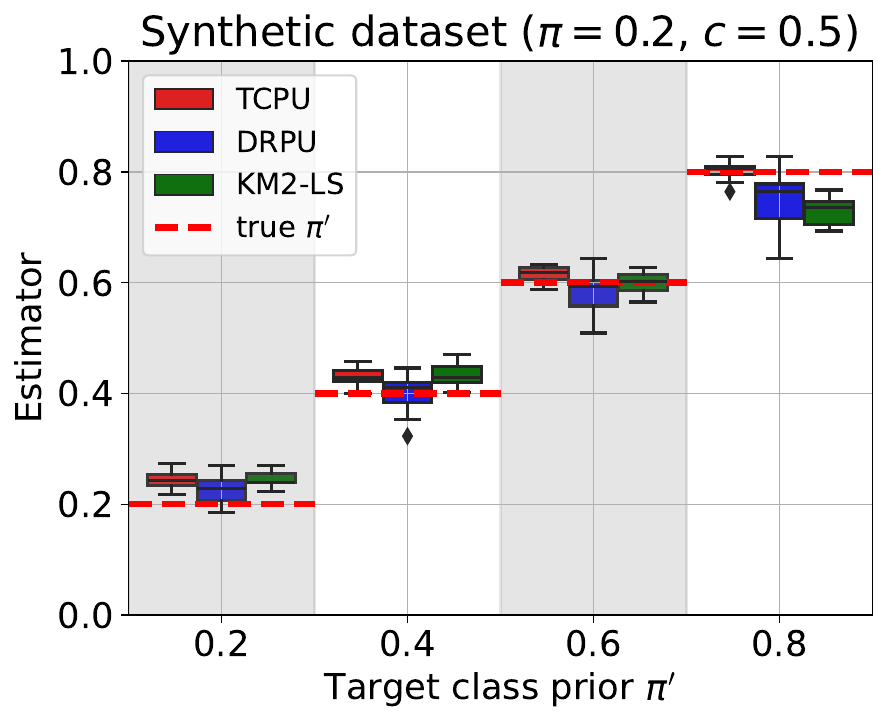}  \\
          \includegraphics[width=0.45\textwidth]{figures1/exp1_est_art1_c0.25_n_sample2000_pi_train0.5.pdf}&
      \includegraphics[width=0.45\textwidth]{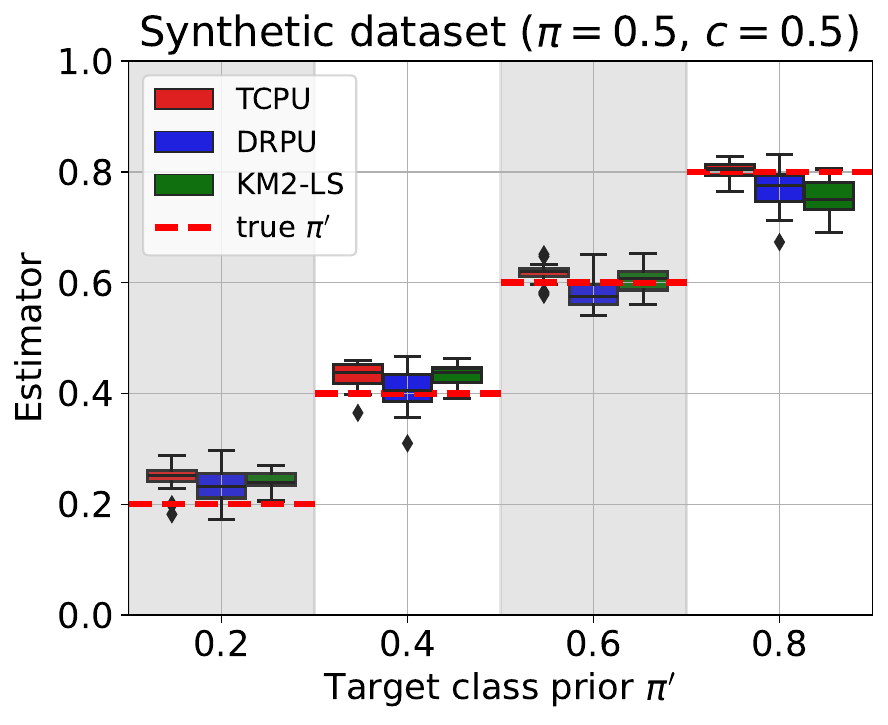}  \\
      \end{tabular}
    \caption{Distribution of estimators (red line indicates  the true $\pi'$). Size of the source data and the target data is $2000$.}
    \label{boxplot_artificial_1}
\end{figure}

\begin{figure}[ht!]
\centering
    \begin{tabular}{c c}
    \includegraphics[width=0.45\textwidth]{figures1/exp1_est_MNIST_c0.25_n_sample2000_pi_train0.1.pdf}  &
      \includegraphics[width=0.45\textwidth]{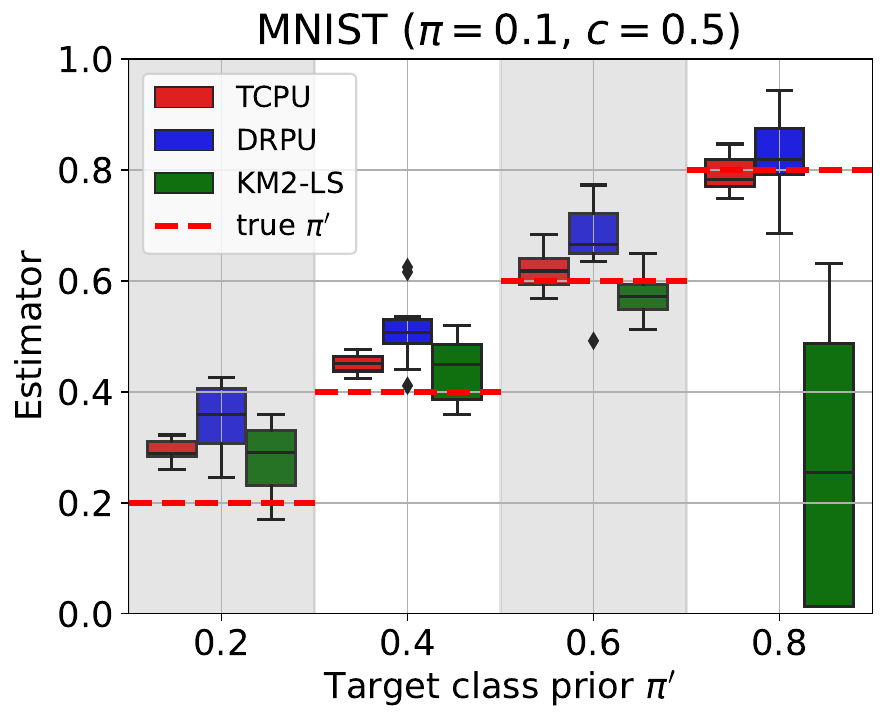}  \\
    \includegraphics[width=0.45\textwidth]{figures1/exp1_est_MNIST_c0.25_n_sample2000_pi_train0.2.pdf}  &
      \includegraphics[width=0.45\textwidth]{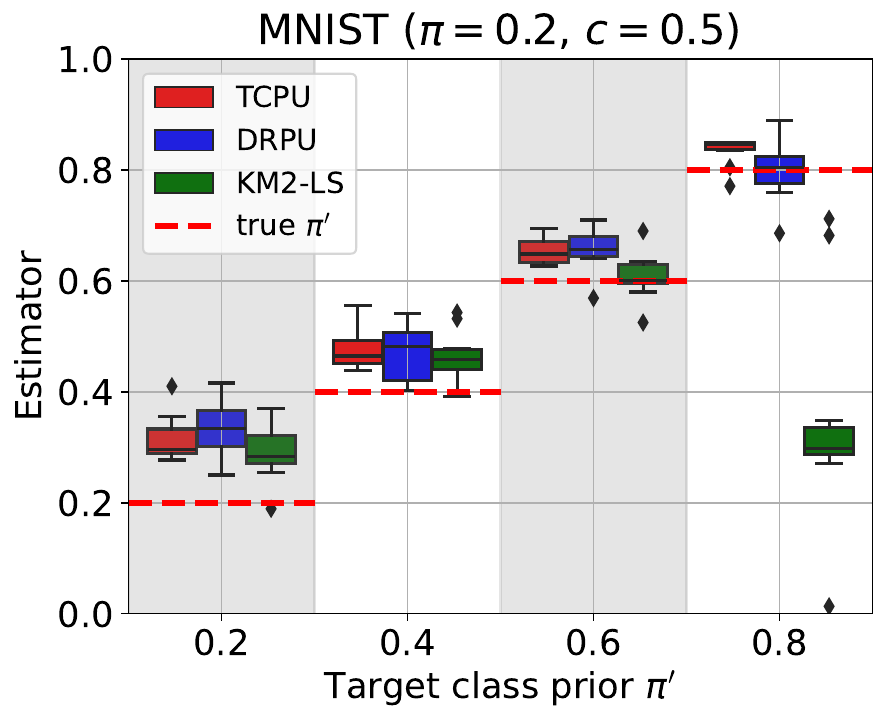}        \\
    \includegraphics[width=0.45\textwidth]{figures1/exp1_est_MNIST_c0.25_n_sample2000_pi_train0.5.pdf}  &
      \includegraphics[width=0.45\textwidth]{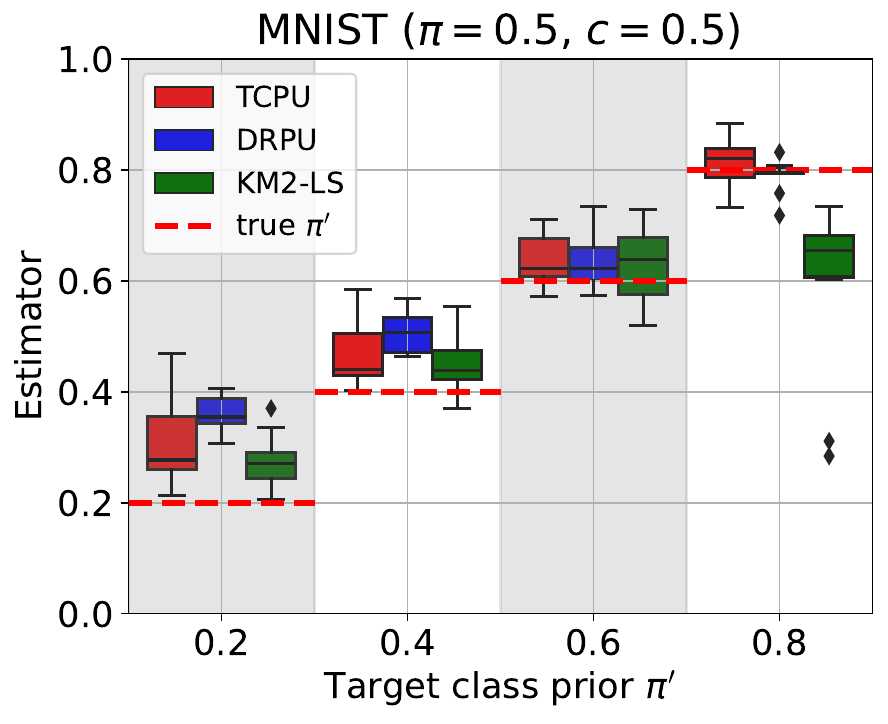}        \\
      \end{tabular}
    \caption{Distribution of estimators (red line indicates  the true $\pi'$) for MNIST dataset.}
    \label{boxplot_mnist}
\end{figure}

\begin{figure}[ht!]
\centering
    \begin{tabular}{c c}
    \includegraphics[width=0.45\textwidth]{figures1/exp1_est_CIFAR10_c0.25_n_sample2000_pi_train0.1.pdf}  &
      \includegraphics[width=0.45\textwidth]{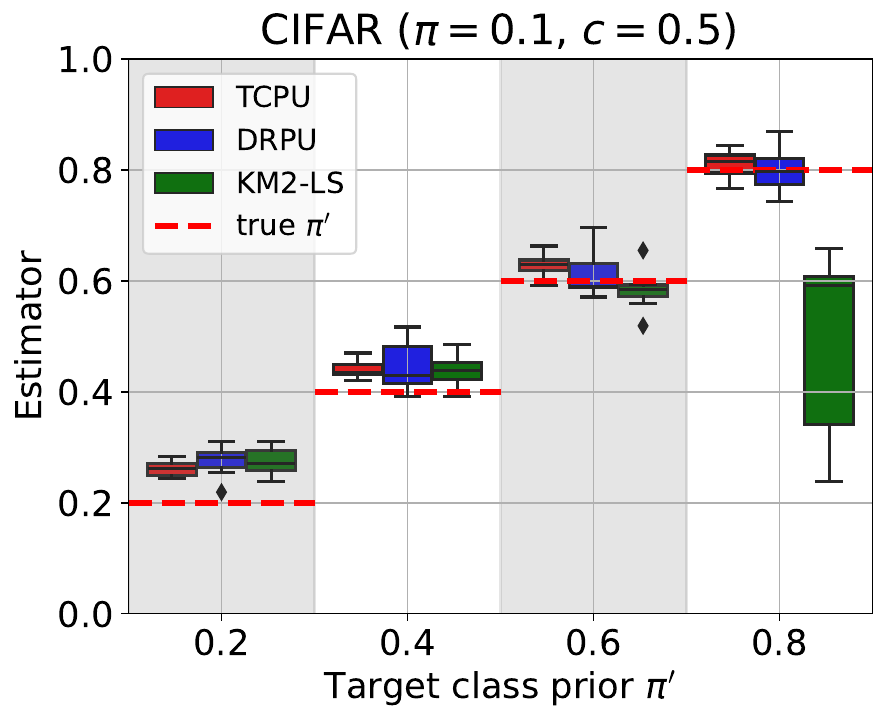}  \\
    \includegraphics[width=0.45\textwidth]{figures1/exp1_est_CIFAR10_c0.25_n_sample2000_pi_train0.2.pdf}  &
      \includegraphics[width=0.45\textwidth]{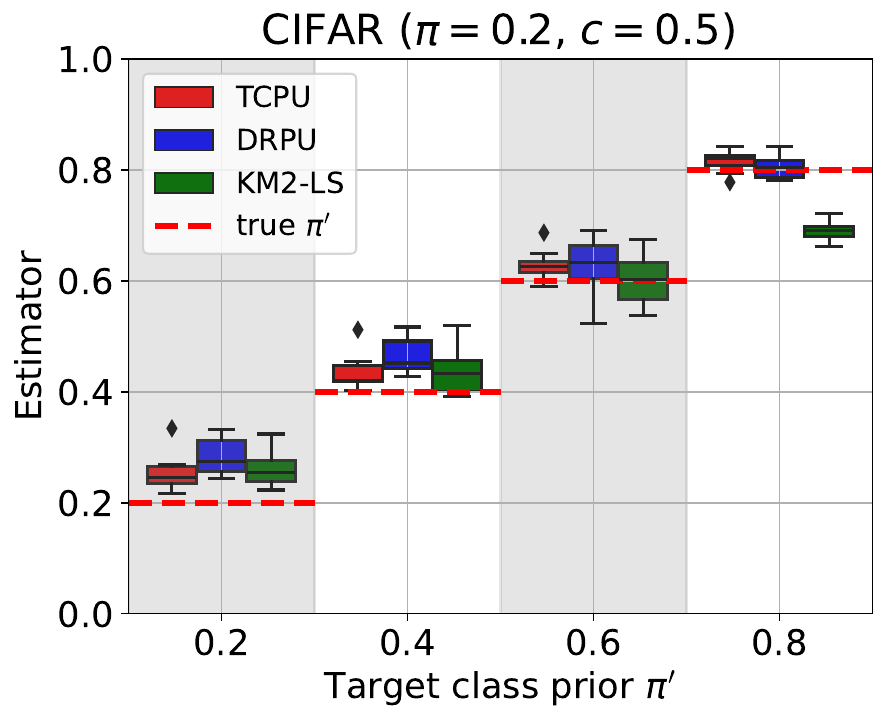}        \\
    \includegraphics[width=0.45\textwidth]{figures1/exp1_est_CIFAR10_c0.25_n_sample2000_pi_train0.5.pdf}  &
      \includegraphics[width=0.45\textwidth]{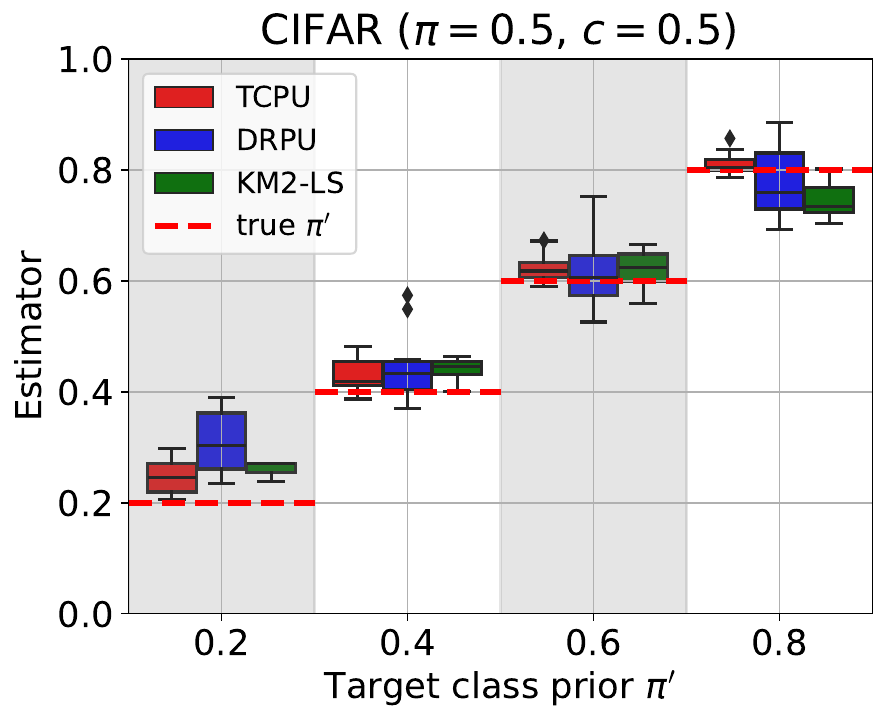}        \\
      \end{tabular}
    \caption{Distribution of estimators (red line indicates  the true $\pi'$) for CIFAR dataset.}
    \label{boxplot_cifar}
\end{figure}

\begin{figure}[ht!]
\centering
    \begin{tabular}{c c}
    \includegraphics[width=0.45\textwidth]{figures1/exp1_est_Fashion_c0.25_n_sample2000_pi_train0.1.pdf}  &
      \includegraphics[width=0.45\textwidth]{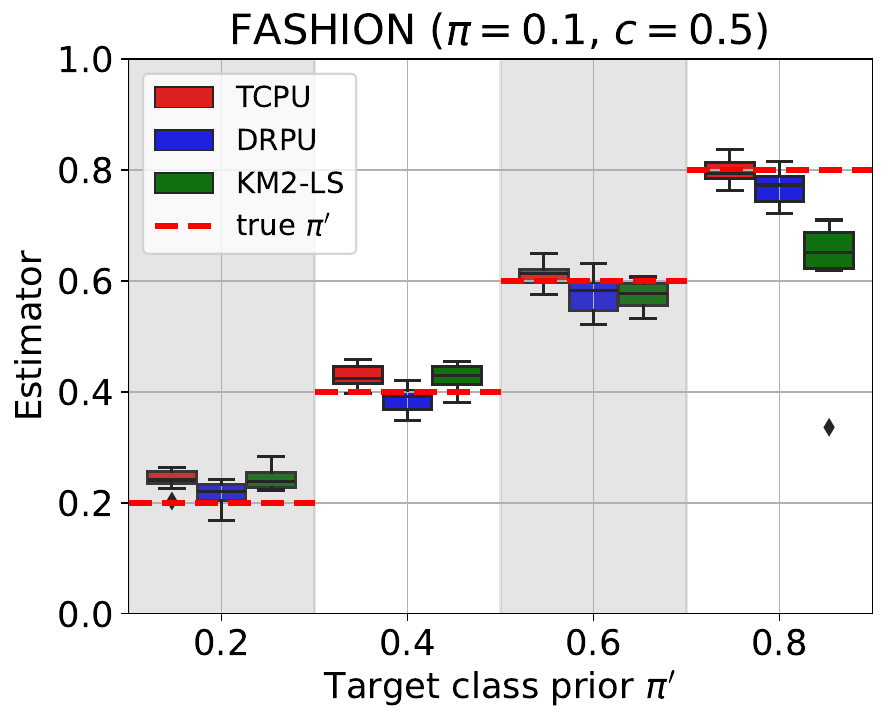}  \\
    \includegraphics[width=0.45\textwidth]{figures1/exp1_est_Fashion_c0.25_n_sample2000_pi_train0.2.pdf}  &
      \includegraphics[width=0.45\textwidth]{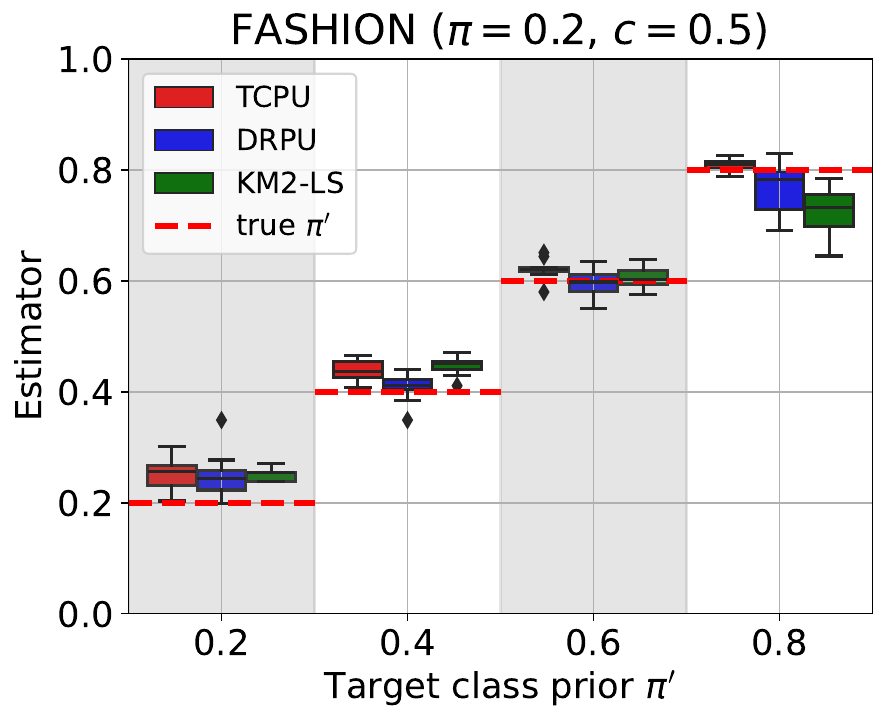}        \\
    \includegraphics[width=0.45\textwidth]{figures1/exp1_est_Fashion_c0.25_n_sample2000_pi_train0.5.pdf}  &
      \includegraphics[width=0.45\textwidth]{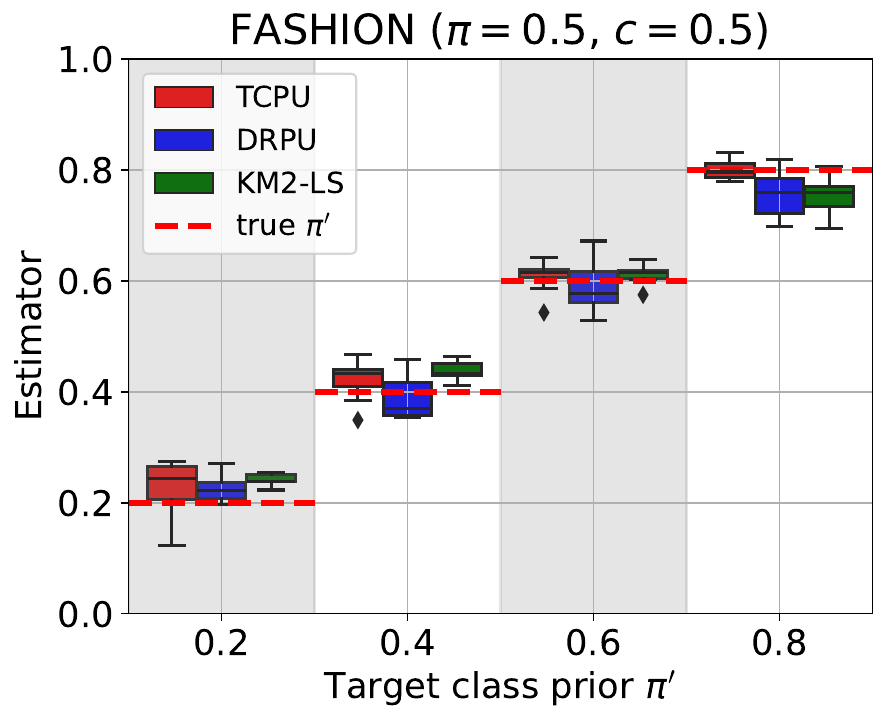}        \\
      \end{tabular}
    \caption{Distribution of estimators (red line indicates  the true $\pi'$) for FASHION dataset.}
    \label{boxplot_fashion}
\end{figure}

\clearpage

\begin{table}[ht!]
\setlength{\heavyrulewidth}{1.5pt}
\begin{center}
\caption{{\bf Estimation errors} for image datasets for $c=0.25$. The winning method and the method whose error does not differ from the winning method by more than $0.01$ are bolded.}
\label{tab:errors_image_c025}
\begin{tabular}{l|l|l|lll}
\toprule
Dataset &	$\pi$ &	$\pi'$ &	TCPU  &	DRPU    &	KM2-LS  \\
\midrule
\multirow{12}{*}{MNIST}&\multirow{4}{*}{0.1}&	0.2 &	0.109 $\pm$ 0.021 &	0.507 $\pm$ 0.027 &	\bf 0.095 $\pm$ 0.026 \\
& &	0.4 &	\bf 0.044 $\pm$ 0.016 &	0.371 $\pm$ 0.019 &	0.065 $\pm$ 0.014 \\
& &	0.6 &	\bf 0.037 $\pm$ 0.008 &	0.209 $\pm$ 0.017 &	0.332 $\pm$ 0.054 \\
& &	0.8 &	\bf 0.058 $\pm$ 0.011 &	0.083 $\pm$ 0.014 &	0.695 $\pm$ 0.045 \\
\cline{2-6}
&\multirow{4}{*}{0.2} &	0.2 &	0.116 $\pm$ 0.015 &	0.207 $\pm$ 0.016 &	\bf 0.089 $\pm$ 0.02 \\
& &	0.4 &	0.067 $\pm$ 0.013 &	0.172 $\pm$ 0.013 &	\bf 0.056 $\pm$ 0.015 \\
& &	0.6 &	\bf 0.034 $\pm$ 0.011 &	0.094 $\pm$ 0.016 &	0.062 $\pm$ 0.015 \\
& &	0.8 &	\bf 0.024 $\pm$ 0.006 &	0.047 $\pm$ 0.015 &	0.702 $\pm$ 0.033 \\
\cline{2-6}
&\multirow{4}{*}{0.5} &	0.2 &	0.13 $\pm$ 0.027 &	0.263 $\pm$ 0.023 &	\bf 0.083 $\pm$ 0.019 \\
& &	0.4 &	0.082 $\pm$ 0.016 &	0.172 $\pm$ 0.021 &	\bf 0.064 $\pm$ 0.012 \\
& &	0.6 &	\bf 0.05 $\pm$ 0.009 &	0.143 $\pm$ 0.02 &	\bf 0.053 $\pm$ 0.013 \\
& &	0.8 &	\bf 0.032 $\pm$ 0.005 &	0.046 $\pm$ 0.006 &	0.396 $\pm$ 0.062 \\
\bottomrule
\multirow{12}{*}{CIFAR}&\multirow{4}{*}{0.1} &	0.2 &	\bf 0.06 $\pm$ 0.005 &	0.334 $\pm$ 0.029 &	\bf 0.059 $\pm$ 0.007 \\
& &	0.4 &	\bf 0.035 $\pm$ 0.003 &	0.262 $\pm$ 0.038 &	\bf 0.029 $\pm$ 0.006 \\
& &	0.6 &	\bf 0.027 $\pm$ 0.004 &	0.183 $\pm$ 0.037 &	0.104 $\pm$ 0.01 \\
& &	0.8 &	\bf 0.022 $\pm$ 0.006 &	0.073 $\pm$ 0.01 &	0.666 $\pm$ 0.025 \\
\cline{2-6}
&\multirow{4}{*}{0.2} &	0.2 &	\bf 0.046 $\pm$ 0.013 &	0.126 $\pm$ 0.016 &	0.057 $\pm$ 0.013 \\
& &	0.4 &	\bf 0.033 $\pm$ 0.01 &	0.101 $\pm$ 0.019 &	\bf 0.036 $\pm$ 0.009 \\
& &	0.6 &	\bf 0.026 $\pm$ 0.008 &	0.07 $\pm$ 0.014 &	0.074 $\pm$ 0.011 \\
& &	0.8 &	\bf 0.03 $\pm$ 0.005 &	0.049 $\pm$ 0.011 &	0.39 $\pm$ 0.053 \\
\cline{2-6}
&\multirow{4}{*}{0.5} &	0.2 &	\bf 0.06 $\pm$ 0.012 &	0.169 $\pm$ 0.031 &	\bf 0.062 $\pm$ 0.007 \\
& &	0.4 &	\bf 0.04 $\pm$ 0.009 &	0.115 $\pm$ 0.021 &	\bf 0.04 $\pm$ 0.006 \\
& &	0.6 &	\bf 0.03 $\pm$ 0.008 &	0.067 $\pm$ 0.016 &	\bf 0.027 $\pm$ 0.008 \\
& &	0.8 &	\bf 0.02 $\pm$ 0.005 &	0.041 $\pm$ 0.007 &	0.135 $\pm$ 0.008 \\
\bottomrule
\multirow{12}{*}{FASHION}&\multirow{4}{*}{0.1} &	0.2 &	\bf 0.04 $\pm$ 0.008 &	0.149 $\pm$ 0.024 &	\bf 0.042 $\pm$ 0.008 \\
& &	0.4 &	\bf 0.03 $\pm$ 0.006 &	0.085 $\pm$ 0.015 &	\bf 0.032 $\pm$ 0.004 \\
& &	0.6 &	\bf 0.02 $\pm$ 0.004 &	0.055 $\pm$ 0.011 &	0.074 $\pm$ 0.011 \\
& &	0.8 &	\bf 0.018 $\pm$ 0.006 &	0.043 $\pm$ 0.01 &	0.56 $\pm$ 0.036 \\
\cline{2-6}
&\multirow{4}{*}{0.2} &	0.2 &	0.054 $\pm$ 0.009 &	\bf 0.042 $\pm$ 0.008 &	0.053 $\pm$ 0.007 \\
& &	0.4 &	\bf 0.039 $\pm$ 0.008 &	0.047 $\pm$ 0.012 &	\bf 0.037 $\pm$ 0.008 \\
& &	0.6 &	\bf 0.025 $\pm$ 0.007 &	\bf 0.024 $\pm$ 0.006 &	0.042 $\pm$ 0.007 \\
& &	0.8 &	\bf 0.019 $\pm$ 0.004 &	\bf 0.025 $\pm$ 0.006 &	0.289 $\pm$ 0.048 \\
\cline{2-6}
&\multirow{4}{*}{0.5} &	0.2 &	\bf 0.042 $\pm$ 0.008 &	0.084 $\pm$ 0.011 &	\bf 0.042 $\pm$ 0.007 \\
& &	0.4 &	\bf 0.03 $\pm$ 0.007 &	0.045 $\pm$ 0.012 &	\bf 0.037 $\pm$ 0.007 \\
& &	0.6 &	\bf 0.024 $\pm$ 0.007 &	0.044 $\pm$ 0.01 &	\bf 0.016 $\pm$ 0.005 \\
& &	0.8 &	\bf 0.024 $\pm$ 0.006 &	\bf 0.032 $\pm$ 0.006 &	0.119 $\pm$ 0.024 \\
\bottomrule
\end{tabular}
\end{center}
\end{table}

\begin{table}[ht!]
\setlength{\heavyrulewidth}{1.5pt}
\begin{center}
\caption{{\bf Estimation errors} for image datasets for $c=0.5$. The winning method and the method whose error does not differ from the winning method by more than $0.01$ are bolded.}
\label{tab:errors_image_c05}
\begin{tabular}{l|l|l|lll}
\toprule 
Dataset &	$\pi$ &	$\pi'$ &	TCPU  &	DRPU    &	KM2-LS  \\
\bottomrule
\multirow{12}{*}{MNIST}&\multirow{4}{*}{0.1} &	0.2 &	\bf 0.094 $\pm$ 0.006 &	0.148 $\pm$ 0.019 &	\bf 0.085 $\pm$ 0.016 \\
& &	0.4 &	\bf 0.05 $\pm$ 0.005 &	0.114 $\pm$ 0.02 &	\bf 0.059 $\pm$ 0.012 \\
& &	0.6 &	\bf 0.03 $\pm$ 0.007 &	0.094 $\pm$ 0.014 &	\bf 0.039 $\pm$ 0.008 \\
& &	0.8 &	\bf 0.029 $\pm$ 0.004 &	0.06 $\pm$ 0.016 &	0.538 $\pm$ 0.076 \\
\cline{2-6}
&\multirow{4}{*}{0.2} &	0.2 &	0.116 $\pm$ 0.013 &	0.133 $\pm$ 0.015 &	\bf 0.092 $\pm$ 0.014 \\
& &	0.4 &	0.078 $\pm$ 0.011 &	\bf 0.069 $\pm$ 0.015 &	\bf 0.065 $\pm$ 0.013 \\
& &	0.6 &	0.053 $\pm$ 0.007 &	0.064 $\pm$ 0.008 &	\bf 0.029 $\pm$ 0.009 \\
& &	0.8 &	\bf 0.039 $\pm$ 0.004 &	\bf 0.042 $\pm$ 0.011 &	0.45 $\pm$ 0.061 \\
\cline{2-6}
&\multirow{4}{*}{0.5} &	0.2 &	0.107 $\pm$ 0.023 &	0.16 $\pm$ 0.01 &	\bf 0.075 $\pm$ 0.015 \\
& &	0.4 &	\bf 0.068 $\pm$ 0.017 &	0.108 $\pm$ 0.012 &	\bf 0.059 $\pm$ 0.015 \\
& &	0.6 &	\bf 0.047 $\pm$ 0.013 &	\bf 0.045 $\pm$ 0.015 &	0.065 $\pm$ 0.012 \\
& &	0.8 &	0.037 $\pm$ 0.008 &	\bf 0.019 $\pm$ 0.008 &	0.209 $\pm$ 0.048 \\
\bottomrule
\multirow{12}{*}{CIFAR}&\multirow{4}{*}{0.1} &	0.2 &	\bf 0.061 $\pm$ 0.004 &	0.076 $\pm$ 0.008 &	0.073 $\pm$ 0.007 \\
& &	0.4 &	\bf 0.04 $\pm$ 0.005 &	0.046 $\pm$ 0.013 &	\bf 0.042 $\pm$ 0.008 \\
& &	0.6 &	\bf 0.03 $\pm$ 0.005 &	0.03 $\pm$ 0.009 &	\bf 0.029 $\pm$ 0.007 \\
& &	0.8 &	\bf 0.021 $\pm$ 0.004 &	0.033 $\pm$ 0.007 &	0.302 $\pm$ 0.051 \\
\cline{2-6}
&\multirow{4}{*}{0.2} &	0.2 &	\bf 0.053 $\pm$ 0.01 &	0.082 $\pm$ 0.01 &	\bf 0.058 $\pm$ 0.009 \\
& &	0.4 &	\bf 0.035 $\pm$ 0.01 &	0.065 $\pm$ 0.01 &	\bf 0.039 $\pm$ 0.012 \\
& &	0.6 &	\bf 0.031 $\pm$ 0.007 &	0.047 $\pm$ 0.009 &	\bf 0.036 $\pm$ 0.007 \\
& &	0.8 &	\bf 0.022 $\pm$ 0.003 &	\bf 0.018 $\pm$ 0.004 &	0.11 $\pm$ 0.005 \\
\cline{2-6}
&\multirow{4}{*}{0.5} &	0.2 &	\bf 0.049 $\pm$ 0.01 &	0.109 $\pm$ 0.017 &	0.061 $\pm$ 0.004 \\
& &	0.4 &	\bf 0.034 $\pm$ 0.008 &	0.057 $\pm$ 0.017 &	\bf 0.042 $\pm$ 0.006 \\
& &	0.6 &	\bf 0.026 $\pm$ 0.008 &	0.054 $\pm$ 0.015 &	\bf 0.033 $\pm$ 0.006 \\
& &	0.8 &	\bf 0.016 $\pm$ 0.006 &	0.061 $\pm$ 0.009 &	0.056 $\pm$ 0.01 \\
\bottomrule
\multirow{12}{*}{FASHION}&\multirow{4}{*}{0.1} &	0.2 &	0.041 $\pm$ 0.005 &	\bf 0.022 $\pm$ 0.005 &	0.045 $\pm$ 0.006 \\
& &	0.4 &	\bf 0.029 $\pm$ 0.006 &	\bf 0.021 $\pm$ 0.006 &	0.031 $\pm$ 0.005 \\
& &	0.6 &	\bf 0.019 $\pm$ 0.004 &	0.033 $\pm$ 0.008 &	\bf 0.028 $\pm$ 0.007 \\
& &	0.8 &	\bf 0.018 $\pm$ 0.004 &	0.035 $\pm$ 0.008 &	0.171 $\pm$ 0.033 \\
\cline{2-6}
&\multirow{4}{*}{0.2} &	0.2 &	0.051 $\pm$ 0.009 &	\bf 0.049 $\pm$ 0.013 &	\bf 0.05 $\pm$ 0.003 \\
& &	0.4 &	0.039 $\pm$ 0.006 &	\bf 0.02 $\pm$ 0.005 &	0.047 $\pm$ 0.005 \\
& &	0.6 &	0.025 $\pm$ 0.004 &	\bf 0.021 $\pm$ 0.005 &	\bf 0.015 $\pm$ 0.003 \\
& &	0.8 &	\bf 0.012 $\pm$ 0.002 &	0.04 $\pm$ 0.011 &	0.073 $\pm$ 0.013 \\
\cline{2-6}
&\multirow{4}{*}{0.5} &	0.2 &	0.049 $\pm$ 0.008 &	\bf 0.025 $\pm$ 0.006 &	0.041 $\pm$ 0.004 \\
& &	0.4 &	\bf 0.036 $\pm$ 0.005 &	\bf 0.035 $\pm$ 0.004 &	0.038 $\pm$ 0.005 \\
& &	0.6 &	\bf 0.023 $\pm$ 0.005 &	0.037 $\pm$ 0.007 &	\bf 0.016 $\pm$ 0.003 \\
& &	0.8 &	\bf 0.014 $\pm$ 0.003 &	0.05 $\pm$ 0.01 &	0.048 $\pm$ 0.009 \\
\bottomrule
\end{tabular}
\end{center}
\end{table}

\begin{table}[ht!]
\setlength{\heavyrulewidth}{1.5pt}
\begin{center}
\caption{{\bf Estimation errors} for benchmark datasets for $c=0.5$ (part 1). The winning method and the method whose error does not differ from the winning method by more than $0.01$ are bolded.}
\label{tab:errors_uci1}
\begin{tabular}{l|l|l|lll}
\toprule 
Dataset &	$\pi$ &	$\pi'$ &	TCPU  &	DRPU   &	KM2-LS \\
\midrule
\multirow{12}{*}{Diabetes}&\multirow{4}{*}{0.1} &	0.2 &	\bf 0.072 $\pm$ 0.013 &	0.2 $\pm$ 0.0 &	\bf 0.071 $\pm$ 0.023 \\
& &	0.4 &	\bf 0.055 $\pm$ 0.019 &	0.4 $\pm$ 0.0 &	0.149 $\pm$ 0.018 \\
& &	0.6 &	\bf 0.104 $\pm$ 0.014 &	0.6 $\pm$ 0.0 &	0.404 $\pm$ 0.021 \\
& &	0.8 &	\bf 0.157 $\pm$ 0.014 &	0.8 $\pm$ 0.0 &	0.771 $\pm$ 0.015 \\
\cline{2-6}
&\multirow{4}{*}{0.2} &	0.2 &	0.129 $\pm$ 0.011 &	0.239 $\pm$ 0.037 &	\bf 0.103 $\pm$ 0.011 \\
& &	0.4 &	\bf 0.072 $\pm$ 0.013 &	0.417 $\pm$ 0.016 &	0.088 $\pm$ 0.015 \\
& &	0.6 &	\bf 0.044 $\pm$ 0.011 &	0.563 $\pm$ 0.035 &	0.281 $\pm$ 0.017 \\
& &	0.8 &	\bf 0.065 $\pm$ 0.016 &	0.735 $\pm$ 0.061 &	0.627 $\pm$ 0.052 \\
\cline{2-6}
&\multirow{4}{*}{0.5} &	0.2 &	\bf 0.121 $\pm$ 0.027 &	0.576 $\pm$ 0.018 &	0.132 $\pm$ 0.011 \\
& &	0.4 &	\bf 0.058 $\pm$ 0.011 &	0.428 $\pm$ 0.014 &	\bf 0.055 $\pm$ 0.014 \\
& &	0.6 &	\bf 0.043 $\pm$ 0.009 &	0.293 $\pm$ 0.014 &	0.173 $\pm$ 0.019 \\
& &	0.8 &	\bf 0.045 $\pm$ 0.011 &	0.138 $\pm$ 0.012 &	0.645 $\pm$ 0.056 \\
\bottomrule
\multirow{12}{*}{Spambase}&\multirow{4}{*}{0.1} &	0.2 &	\bf 0.021 $\pm$ 0.006 &	0.067 $\pm$ 0.013 &	\bf 0.024 $\pm$ 0.004 \\
& &	0.4 &	\bf 0.02 $\pm$ 0.005 &	0.061 $\pm$ 0.011 &	0.078 $\pm$ 0.015 \\
& &	0.6 &	\bf 0.021 $\pm$ 0.007 &	0.044 $\pm$ 0.008 &	0.299 $\pm$ 0.022 \\
& &	0.8 &	\bf 0.031 $\pm$ 0.008 &	\bf 0.036 $\pm$ 0.01 &	0.592 $\pm$ 0.012 \\
\cline{2-6}
&\multirow{4}{*}{0.2} &	0.2 &	\bf 0.023 $\pm$ 0.005 &	0.046 $\pm$ 0.007 &	\bf 0.02 $\pm$ 0.003 \\
& &	0.4 &	\bf 0.021 $\pm$ 0.004 &	0.05 $\pm$ 0.01 &	0.039 $\pm$ 0.006 \\
& &	0.6 &	\bf 0.014 $\pm$ 0.003 &	0.037 $\pm$ 0.01 &	0.161 $\pm$ 0.024 \\
& &	0.8 &	\bf 0.016 $\pm$ 0.003 &	0.029 $\pm$ 0.003 &	0.469 $\pm$ 0.02 \\
\cline{2-6}
&\multirow{4}{*}{0.5} &	0.2 &	\bf 0.025 $\pm$ 0.01 &	0.057 $\pm$ 0.015 &	\bf 0.022 $\pm$ 0.005 \\
& &	0.4 &	\bf 0.025 $\pm$ 0.005 &	0.059 $\pm$ 0.012 &	\bf 0.02 $\pm$ 0.005 \\
& &	0.6 &	\bf 0.015 $\pm$ 0.004 &	\bf 0.023 $\pm$ 0.005 &	0.102 $\pm$ 0.018 \\
& &	0.8 &	\bf 0.017 $\pm$ 0.004 &	0.046 $\pm$ 0.012 &	0.378 $\pm$ 0.026 \\
\bottomrule
\multirow{12}{*}{Segment} &\multirow{4}{*}{0.1} &	0.2 &	0.019 $\pm$ 0.004 &	\bf 0.007 $\pm$ 0.002 &	0.032 $\pm$ 0.008 \\
& &	0.4 &	\bf 0.022 $\pm$ 0.004 &	\bf 0.019 $\pm$ 0.004 &	0.12 $\pm$ 0.013 \\
& &	0.6 &	\bf 0.021 $\pm$ 0.006 &	\bf 0.029 $\pm$ 0.007 &	0.18 $\pm$ 0.028 \\
& &	0.8 &	\bf 0.022 $\pm$ 0.004 &	0.042 $\pm$ 0.01 &	0.499 $\pm$ 0.025 \\
\cline{2-6}
&\multirow{4}{*}{0.2} &	0.2 &	\bf 0.017 $\pm$ 0.005 &	\bf 0.008 $\pm$ 0.003 &	0.036 $\pm$ 0.006 \\
& &	0.4 &	\bf 0.022 $\pm$ 0.004 &	\bf 0.022 $\pm$ 0.006 &	0.114 $\pm$ 0.013 \\
& &	0.6 &	\bf 0.02 $\pm$ 0.005 &	0.037 $\pm$ 0.013 &	0.182 $\pm$ 0.032 \\
& &	0.8 &	\bf 0.019 $\pm$ 0.003 &	0.047 $\pm$ 0.013 &	0.504 $\pm$ 0.017 \\
\cline{2-6}
&\multirow{4}{*}{0.5} &	0.2 &	0.101 $\pm$ 0.021 &	\bf 0.028 $\pm$ 0.008 &	0.052 $\pm$ 0.008 \\
& &	0.4 &	0.086 $\pm$ 0.021 &	\bf 0.043 $\pm$ 0.012 &	0.121 $\pm$ 0.01 \\
& &	0.6 &	0.052 $\pm$ 0.012 &	\bf 0.039 $\pm$ 0.015 &	0.194 $\pm$ 0.033 \\
& &	0.8 &	\bf 0.026 $\pm$ 0.006 &	0.041 $\pm$ 0.013 &	0.493 $\pm$ 0.024 \\
\bottomrule
\multirow{12}{*}{Waveform} &\multirow{4}{*}{0.1} &	0.2 &	0.156 $\pm$ 0.006 &	\bf 0.105 $\pm$ 0.01 &	0.165 $\pm$ 0.006 \\
& &	0.4 &	0.107 $\pm$ 0.006 &	0.101 $\pm$ 0.018 &	\bf 0.091 $\pm$ 0.008 \\
& &	0.6 &	0.047 $\pm$ 0.005 &	0.08 $\pm$ 0.012 &	\bf 0.026 $\pm$ 0.007 \\
& &	0.8 &	\bf 0.016 $\pm$ 0.004 &	0.039 $\pm$ 0.007 &	0.445 $\pm$ 0.027 \\
\cline{2-6}
&\multirow{4}{*}{0.2} &	0.2 &	\bf 0.165 $\pm$ 0.005 &	0.215 $\pm$ 0.015 &	0.176 $\pm$ 0.006 \\
& &	0.4 &	0.117 $\pm$ 0.007 &	0.156 $\pm$ 0.014 &	\bf 0.103 $\pm$ 0.007 \\
& &	0.6 &	0.063 $\pm$ 0.004 &	0.104 $\pm$ 0.016 &	\bf 0.026 $\pm$ 0.006 \\
& &	0.8 &	\bf 0.017 $\pm$ 0.003 &	0.07 $\pm$ 0.007 &	0.525 $\pm$ 0.092 \\
\cline{2-6}
&\multirow{4}{*}{0.5} &	0.2 &	\bf 0.169 $\pm$ 0.005 &	0.187 $\pm$ 0.02 &	\bf 0.169 $\pm$ 0.008 \\
& &	0.4 &	0.123 $\pm$ 0.006 &	0.156 $\pm$ 0.02 &	\bf 0.107 $\pm$ 0.007 \\
& &	0.6 &	0.07 $\pm$ 0.004 &	0.087 $\pm$ 0.009 &	\bf 0.036 $\pm$ 0.006 \\
& &	0.8 &	\bf 0.017 $\pm$ 0.004 &	0.054 $\pm$ 0.01 &	0.19 $\pm$ 0.008 \\
\bottomrule
\end{tabular}
\end{center}
\end{table}

\begin{table}[ht!]
\setlength{\heavyrulewidth}{1.5pt}
\begin{center}
\caption{{\bf Estimation errors} for benchmark datasets for $c=0.5$ (part 2). The winning method and the method whose error does not differ from the winning method by more than $0.01$ are bolded.}
\label{tab:errors_uci2}
\begin{tabular}{l|l|l|lll}
\toprule 
Dataset &	$\pi$ &	$\pi'$ &	TCPU  &	DRPU    &	KM2-LS  \\
\midrule
\multirow{12}{*}{Yeast} &\multirow{4}{*}{0.1} &	0.2 &	0.087 $\pm$ 0.024 &	0.2 $\pm$ 0.0 &	\bf 0.057 $\pm$ 0.019 \\
& &	0.4 &	\bf 0.043 $\pm$ 0.016 &	0.4 $\pm$ 0.0 &	0.164 $\pm$ 0.015 \\
& &	0.6 &	\bf 0.072 $\pm$ 0.017 &	0.6 $\pm$ 0.0 &	0.423 $\pm$ 0.027 \\
& &	0.8 &	\bf 0.102 $\pm$ 0.016 &	0.8 $\pm$ 0.0 &	0.723 $\pm$ 0.032 \\
\cline{2-6}
&\multirow{4}{*}{0.1} &	0.2 &	0.194 $\pm$ 0.045 &	0.462 $\pm$ 0.023 &	\bf 0.141 $\pm$ 0.029 \\
& &	0.4 &	0.142 $\pm$ 0.028 &	0.37 $\pm$ 0.015 &	\bf 0.1 $\pm$ 0.015 \\
& &	0.6 &	\bf 0.095 $\pm$ 0.019 &	0.216 $\pm$ 0.017 &	0.272 $\pm$ 0.022 \\
& &	0.8 &	\bf 0.057 $\pm$ 0.015 &	0.102 $\pm$ 0.011 &	0.641 $\pm$ 0.047 \\
\cline{2-6}
&\multirow{4}{*}{0.1} &	0.2 &	0.238 $\pm$ 0.035 &	0.43 $\pm$ 0.025 &	\bf 0.164 $\pm$ 0.042 \\
& &	0.4 &	0.161 $\pm$ 0.022 &	0.356 $\pm$ 0.019 &	\bf 0.146 $\pm$ 0.026 \\
& &	0.6 &	\bf 0.117 $\pm$ 0.016 &	0.22 $\pm$ 0.02 &	0.249 $\pm$ 0.018 \\
& &	0.8 &	\bf 0.055 $\pm$ 0.015 &	0.065 $\pm$ 0.015 &	0.543 $\pm$ 0.04 \\
\bottomrule
\multirow{12}{*}{Vehicle} &\multirow{4}{*}{0.1} &	0.2 &	\bf 0.035 $\pm$ 0.006 &	0.2 $\pm$ 0.0 &	0.07 $\pm$ 0.009 \\
& &	0.4 &	\bf 0.055 $\pm$ 0.012 &	0.4 $\pm$ 0.0 &	0.227 $\pm$ 0.014 \\
& &	0.6 &	\bf 0.065 $\pm$ 0.017 &	0.6 $\pm$ 0.0 &	0.496 $\pm$ 0.031 \\
& &	0.8 &	\bf 0.088 $\pm$ 0.021 &	0.8 $\pm$ 0.0 &	0.729 $\pm$ 0.028 \\
\cline{2-6}
&\multirow{4}{*}{0.2} &	0.2 &	\bf 0.028 $\pm$ 0.006 &	0.099 $\pm$ 0.026 &	0.047 $\pm$ 0.006 \\
& &	0.4 &	\bf 0.042 $\pm$ 0.006 &	0.088 $\pm$ 0.024 &	0.163 $\pm$ 0.009 \\
& &	0.6 &	\bf 0.047 $\pm$ 0.007 &	\bf 0.038 $\pm$ 0.011 &	0.344 $\pm$ 0.029 \\
& &	0.8 &	0.051 $\pm$ 0.011 &	\bf 0.032 $\pm$ 0.007 &	0.655 $\pm$ 0.044 \\
\cline{2-6}
&\multirow{4}{*}{0.5} &	0.2 &	0.088 $\pm$ 0.016 &	0.18 $\pm$ 0.029 &	\bf 0.016 $\pm$ 0.006 \\
& &	0.4 &	\bf 0.078 $\pm$ 0.02 &	0.172 $\pm$ 0.027 &	0.174 $\pm$ 0.015 \\
& &	0.6 &	\bf 0.073 $\pm$ 0.015 &	0.091 $\pm$ 0.015 &	0.318 $\pm$ 0.016 \\
& &	0.8 &	0.078 $\pm$ 0.013 &	\bf 0.066 $\pm$ 0.014 &	0.677 $\pm$ 0.045 \\
\bottomrule
\multirow{12}{*}{Banknote} &\multirow{4}{*}{0.1}&	0.2 &	\bf 0.02 $\pm$ 0.003 &	0.2 $\pm$ 0.0 &	0.063 $\pm$ 0.014 \\
& &	0.4 &	\bf 0.02 $\pm$ 0.003 &	0.4 $\pm$ 0.0 &	0.204 $\pm$ 0.025 \\
& &	0.6 &	\bf 0.045 $\pm$ 0.008 &	0.6 $\pm$ 0.0 &	0.372 $\pm$ 0.028 \\
& &	0.8 &	\bf 0.052 $\pm$ 0.006 &	0.8 $\pm$ 0.0 &	0.584 $\pm$ 0.041 \\
\cline{2-6}
&\multirow{4}{*}{0.2} &	0.2 &	\bf 0.021 $\pm$ 0.005 &	\bf 0.022 $\pm$ 0.004 &	0.04 $\pm$ 0.009 \\
& &	0.4 &	\bf 0.019 $\pm$ 0.004 &	\bf 0.027 $\pm$ 0.012 &	0.11 $\pm$ 0.025 \\
& &	0.6 &	\bf 0.03 $\pm$ 0.009 &	\bf 0.024 $\pm$ 0.008 &	0.309 $\pm$ 0.024 \\
& &	0.8 &	\bf 0.023 $\pm$ 0.005 &	\bf 0.021 $\pm$ 0.007 &	0.487 $\pm$ 0.033 \\
\cline{2-6}
&\multirow{4}{*}{0.5} &	0.2 &	0.07 $\pm$ 0.019 &	\bf 0.016 $\pm$ 0.006 &	0.028 $\pm$ 0.006 \\
& &	0.4 &	\bf 0.05 $\pm$ 0.016 &	\bf 0.041 $\pm$ 0.018 &	0.117 $\pm$ 0.026 \\
& &	0.6 &	0.041 $\pm$ 0.011 &	\bf 0.026 $\pm$ 0.006 &	0.16 $\pm$ 0.033 \\
& &	0.8 &	\bf 0.033 $\pm$ 0.009 &	\bf 0.031 $\pm$ 0.01 &	0.38 $\pm$ 0.048 \\
\bottomrule
\end{tabular}
\end{center}
\end{table}

\begin{figure}[ht!]
\centering
    \begin{tabular}{c c}
    \includegraphics[width=0.45\textwidth]{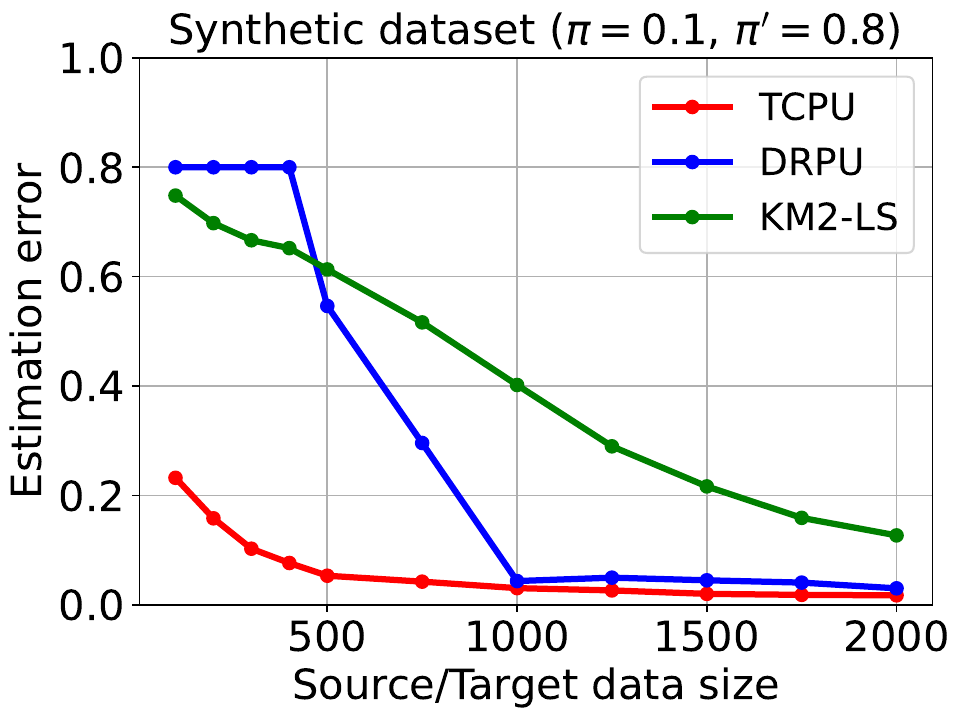}  &
      \includegraphics[width=0.45\textwidth]{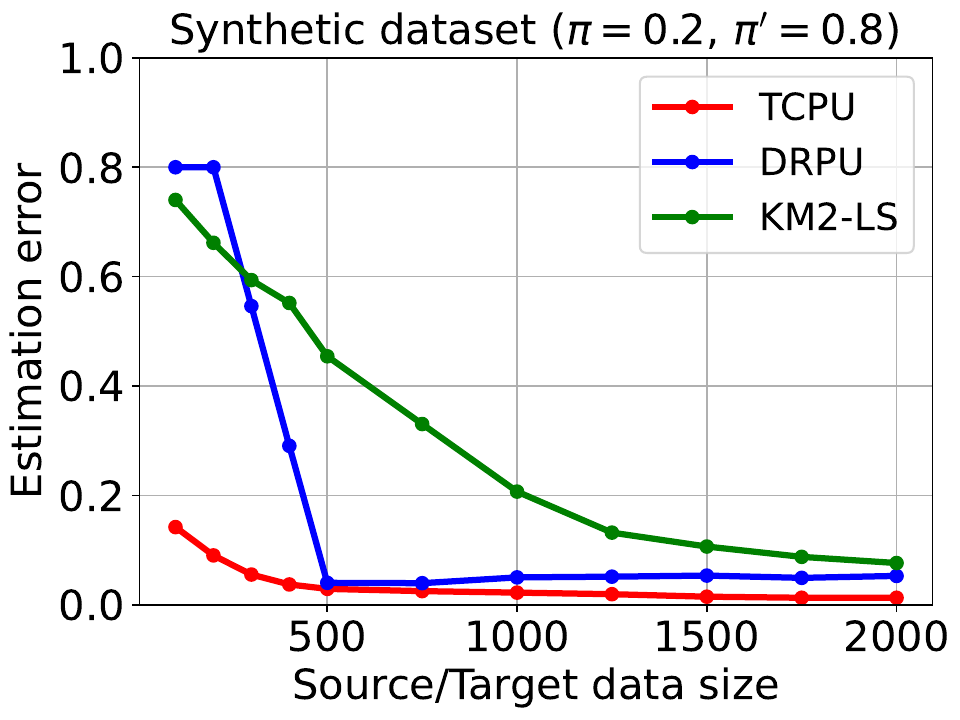}  \\
      \includegraphics[width=0.45\textwidth]{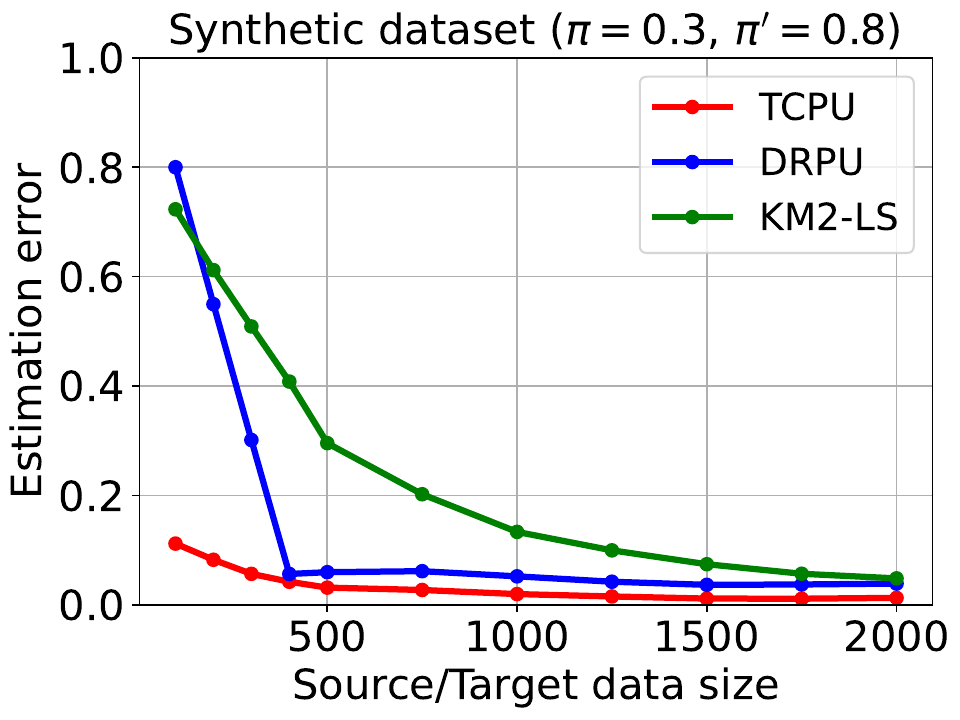}  &
      \includegraphics[width=0.45\textwidth]{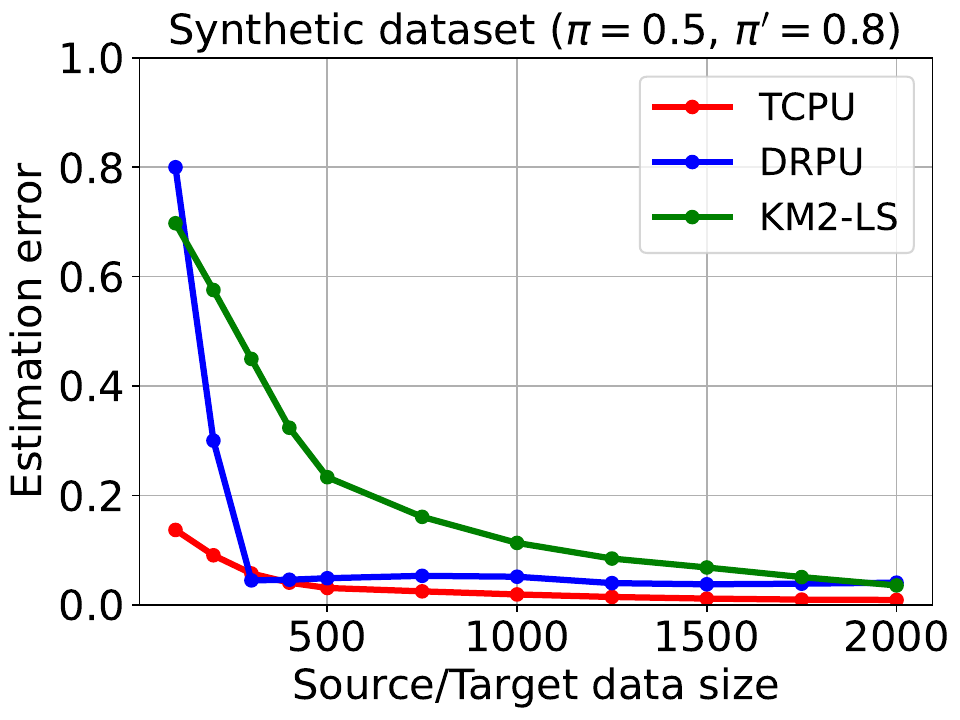}  \\     
      \end{tabular}
    \caption{Estimation errors wrt size of the source data for synthetic dataset and $c=0.5$. We assume that the size of the target data is equal to the size of the source data.}
    \label{lineplot_artificial_n}
\end{figure}

\begin{figure}[ht!]
\centering
    \begin{tabular}{c c}
    \includegraphics[width=0.45\textwidth]{figures1/exp3_error_art1_c0.5_n_sample2000_pi_train0.1_pi_test0.5.pdf}  &
      \includegraphics[width=0.45\textwidth]{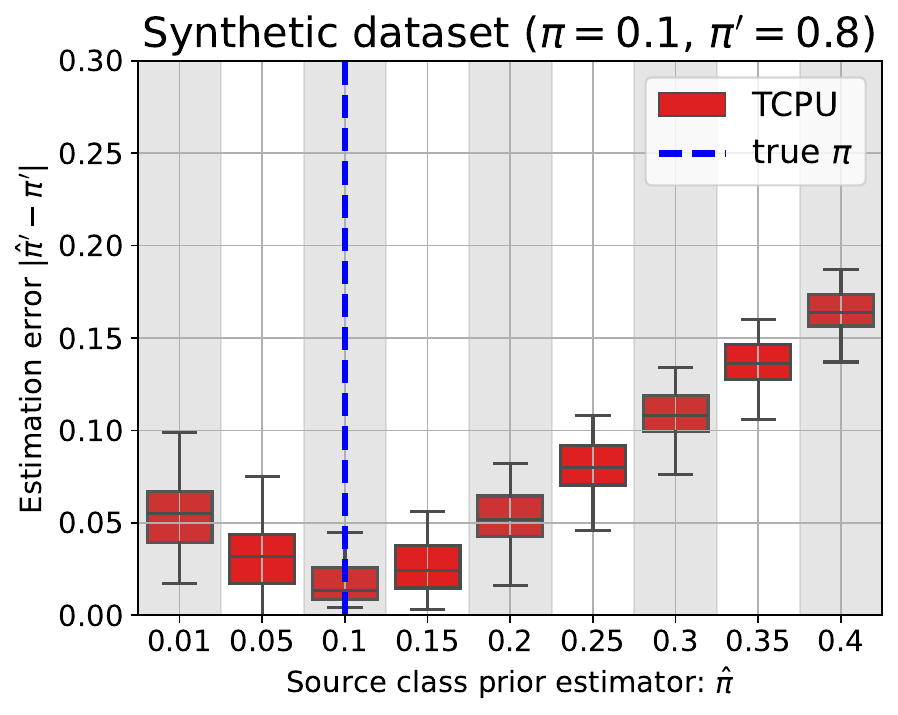}  \\
      \includegraphics[width=0.45\textwidth]{figures1/exp3_error_art1_c0.5_n_sample2000_pi_train0.2_pi_test0.5.pdf}  &
      \includegraphics[width=0.45\textwidth]{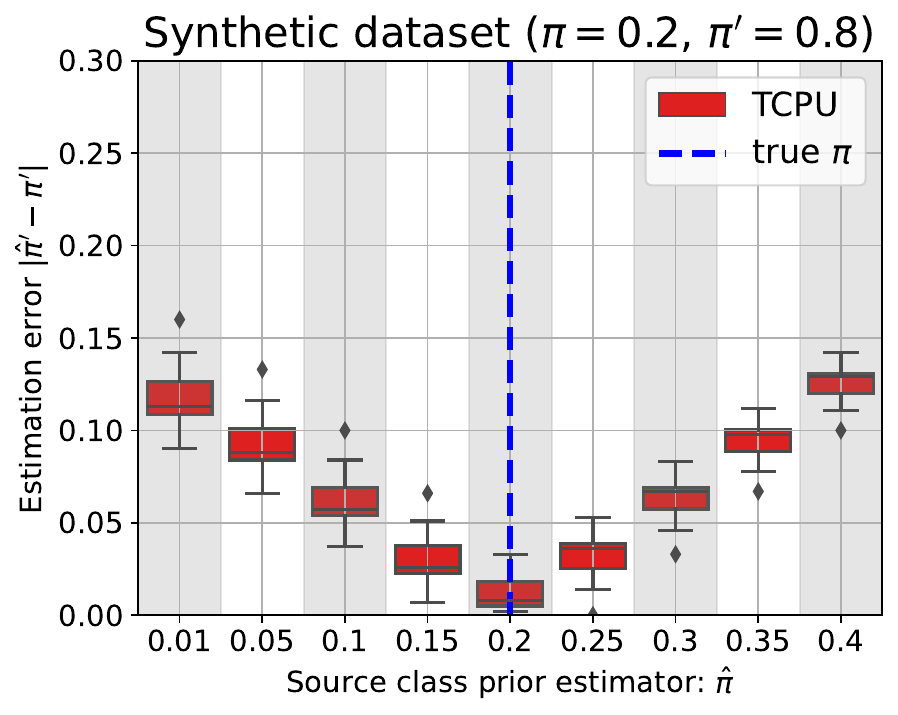}  \\    
    \includegraphics[width=0.45\textwidth]{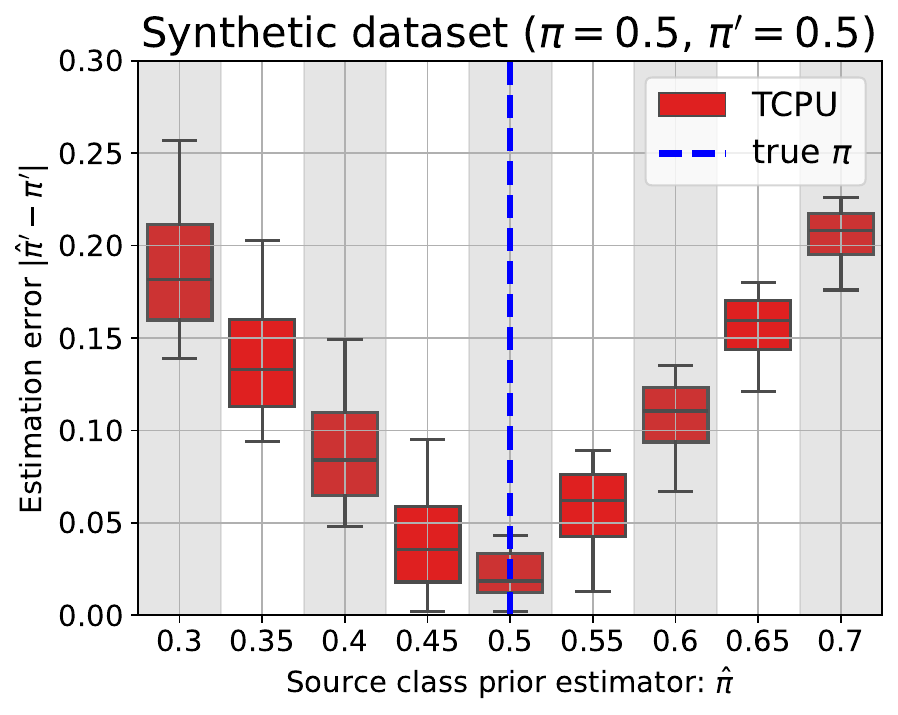}  &
      \includegraphics[width=0.45\textwidth]{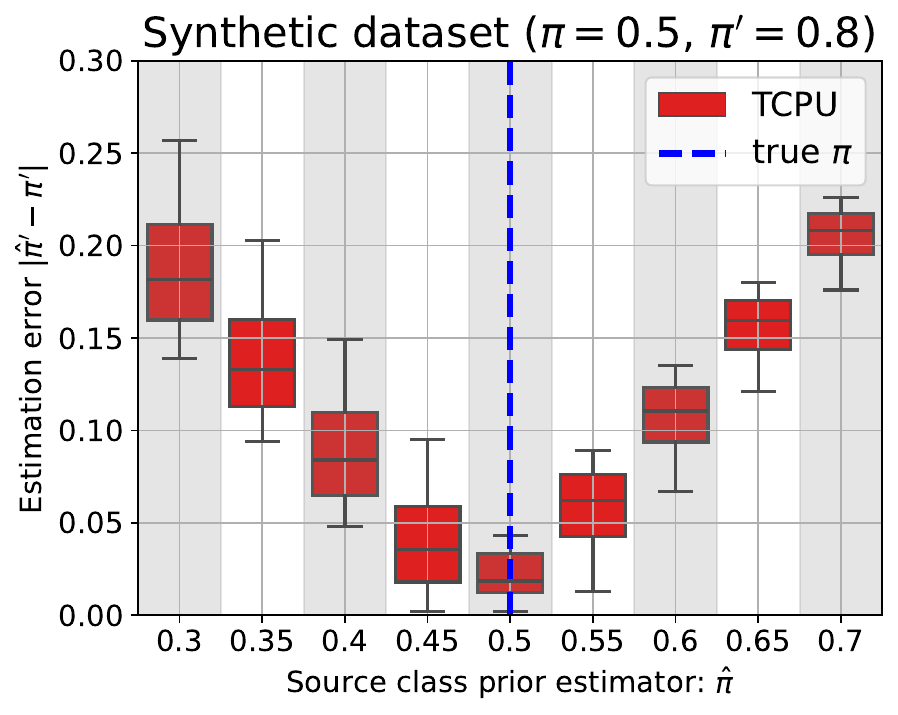}  \\    
 
      \end{tabular}
    \caption{The impact of $\pi$ estimation on the performance of the TCPU estimator. The boxplots show estimation errors for TCPU target class prior estimator $|\hpp-\pi'|$, for different source class prior estimators $\hat{\pi}$.}
    \label{boxplot_artificial_3}
\end{figure}

\begin{figure}[ht!]
\centering
    \begin{tabular}{c c}
    \includegraphics[width=0.45\textwidth]{figures1/exp4_error_art1_c0.5_n_sample2000_pi_train0.1_pi_test0.5.pdf}  &
      \includegraphics[width=0.45\textwidth]{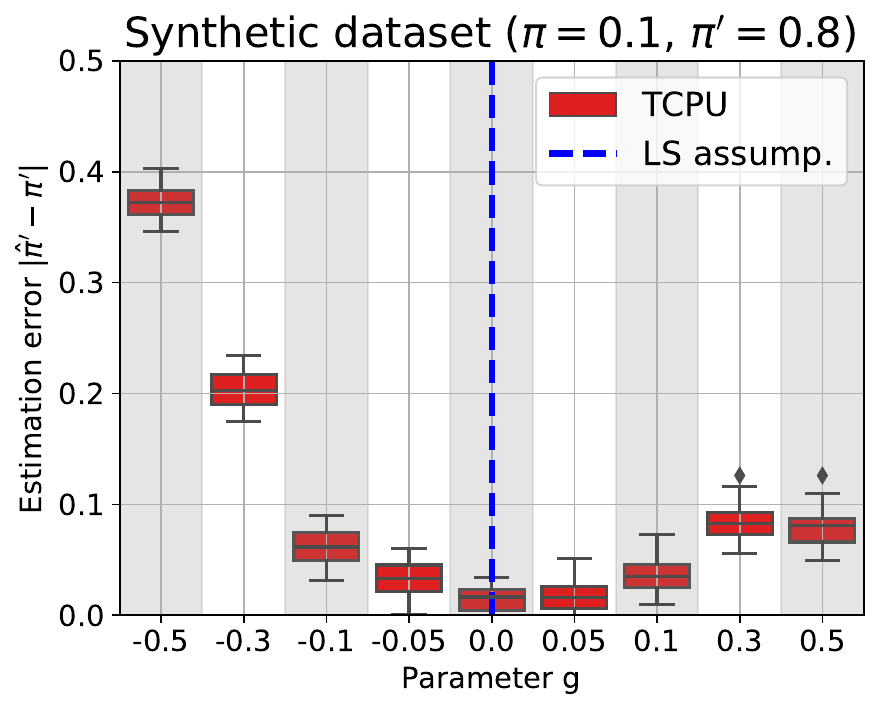}  \\
      \includegraphics[width=0.45\textwidth]{figures1/exp4_error_art1_c0.5_n_sample2000_pi_train0.2_pi_test0.5.pdf}  &
      \includegraphics[width=0.45\textwidth]{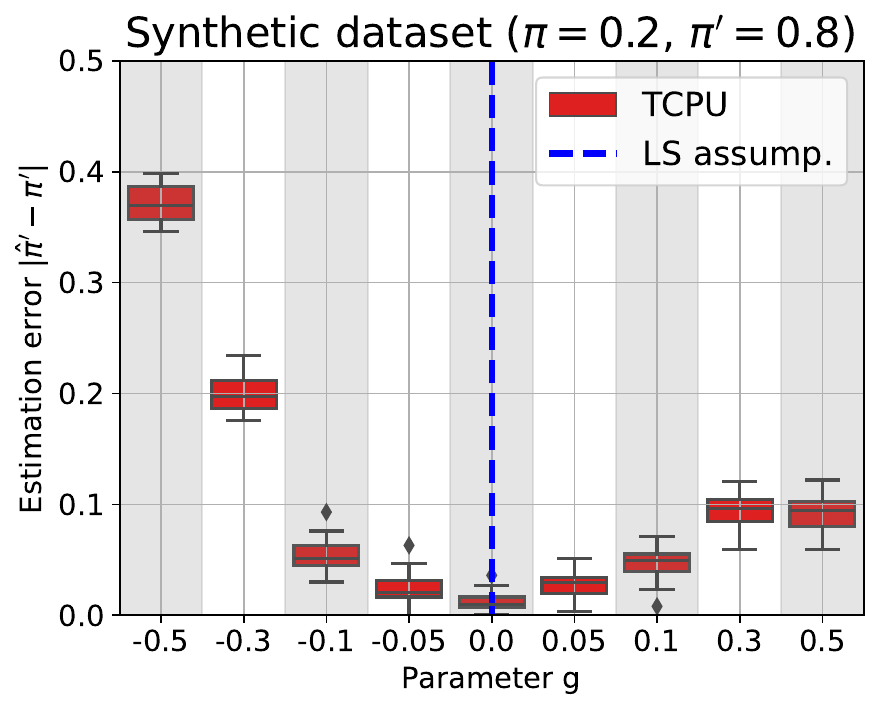}  \\    
    \includegraphics[width=0.45\textwidth]{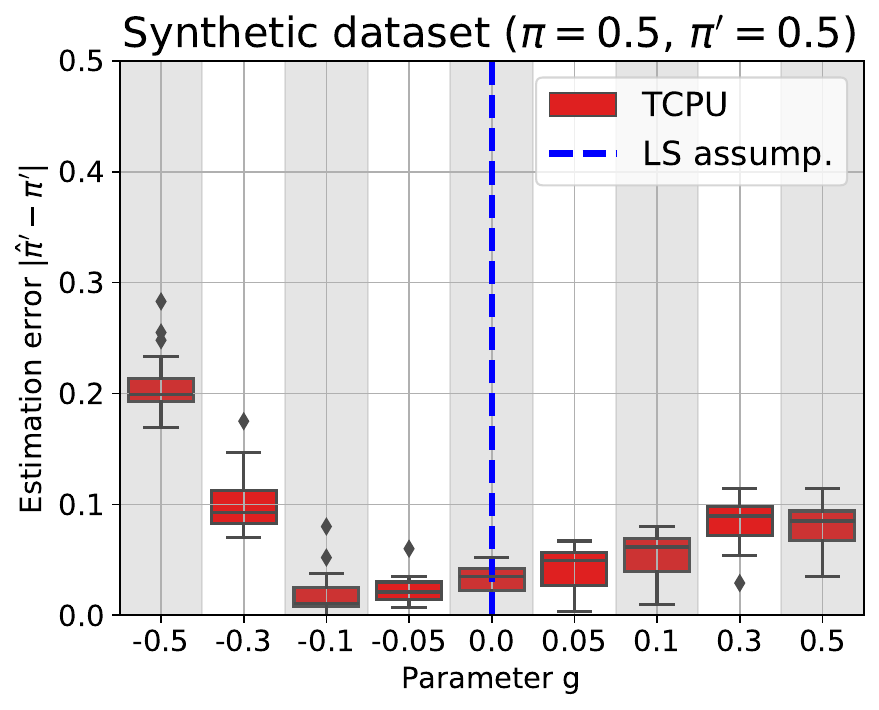}  &
      \includegraphics[width=0.45\textwidth]{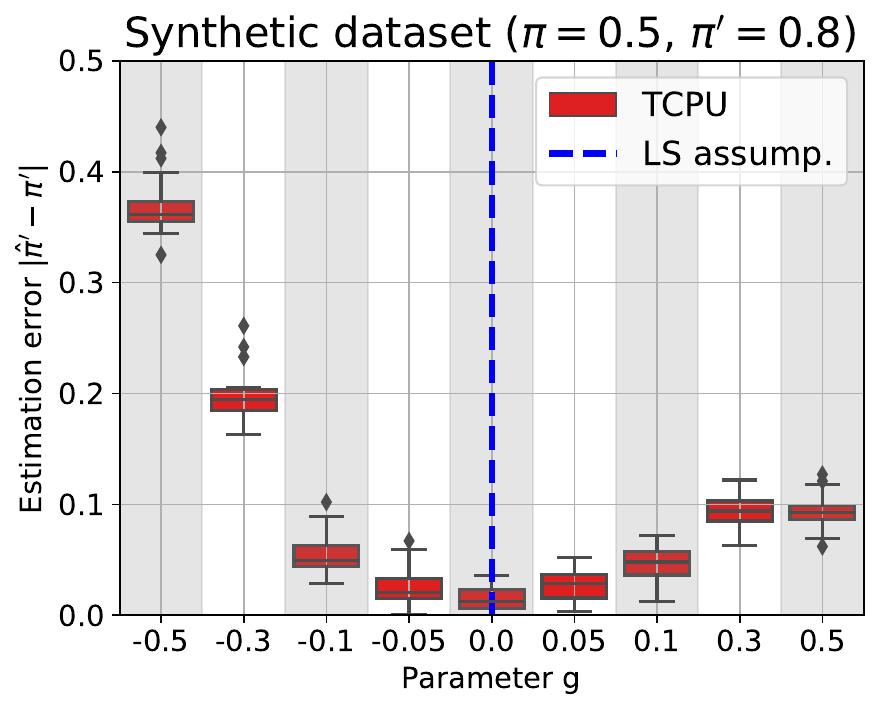}  \\    
      \end{tabular}
    \caption{Robustness to violation of the Label Shift (LS) assumption. The blue vertical line corresponds to the situation when the assumption is met, and non-zero values of the parameter $g$ indicate a violation of the assumption. Boxplots show the distributions of TCPU estimation errors for different values of the parameter $g$.}
    \label{boxplot_artificial_4}
\end{figure}

\end{document}